\documentclass{article}
\usepackage[T1]{fontenc}
\usepackage[utf8]{inputenc}
\usepackage{verbatim}
\usepackage{float}
\usepackage{booktabs}
\usepackage{units}
\usepackage{url}
\usepackage{amsmath}
\usepackage{graphicx}
\usepackage{tablefootnote}
\usepackage{microtype}
\usepackage[bookmarks=false,
 breaklinks=false,pdfborder={0 0 1},backref=section,colorlinks=false]
 {hyperref}

\makeatletter

\providecommand{\tabularnewline}{\\}

\floatstyle{ruled}
\newfloat{algorithm}{tbp}{loa}
\providecommand{\algorithmname}{Algorithm}
\floatname{algorithm}{\protect\algorithmname}

\usepackage[preprint]{neurips_2025}

\usepackage{url}
\usepackage{amsfonts,amsmath,amssymb,amsthm}
\usepackage{nicefrac}
\usepackage[dvipsnames]{xcolor}

\usepackage{algorithm}
\usepackage{algpseudocode}
\newtheorem{theorem}{Theorem}
\theoremstyle{definition}
\newtheorem{definition}{Definition}
\newtheorem{lemma}{Lemma}

\title{Multi-Actor Multi-Critic Deep Deterministic Reinforcement Learning with a Novel Q-Ensemble Method
}

\author{%
  Andy Wu\\
  \texttt{andywu.academic@gmail.com}\\
  \And
  Chun-Cheng Lin\\
  \texttt{cclin@csie.fju.edu.tw}\\
  \And
  Rung-Tzuo Liaw\\
  \texttt{rtliaw@csie.fju.edu.tw}\\
  \And
  Yuehua Huang\\
  \texttt{yhhuang@csie.fju.edu.tw}\\
  \And
  Chihjung Kuo\\
  \texttt{cjkuo@csie.fju.edu.tw}\\
  \And
  Chia Tong Weng\\
  \texttt{ctweng@csie.fju.edu.tw}\\
}

\begin{document}

\makeatother

\maketitle
\begin{abstract}
Reinforcement learning has gathered much attention in recent years
due to its rapid development and rich applications, especially on
control systems and robotics. When tackling real-world applications
with reinforcement learning method, the corresponded Markov decision
process may have huge discrete or even continuous state/action space.
Deep reinforcement learning has been studied for handling these issues
through deep learning for years, and one promising branch is the actor-critic
architecture. Many past studies leveraged multiple critics to enhance
the accuracy of evaluation of a policy for addressing the overestimation
and underestimation issues. However, few studies have considered the
architecture with multiple actors together with multiple critics.
This study proposes a novel multi-actor multi-critic (MAMC) deep deterministic
reinforcement learning method. The proposed method has three main
features, including selection of actors based on non-dominated sorting
for exploration with respect to skill and creativity factors, evaluation
for actors and critics using a quantile-based ensemble strategy, and
exploiting actors with best skill factor. Theoretical analysis proves
the learning stability and bounded estimation bias for the MAMC. The
present study examines the performance on a well-known reinforcement
learning benchmark MuJoCo. Experimental results show that the proposed
framework outperforms state-of-the-art deep deterministic based reinforcement
learning methods. Experimental analysis also indicates the proposed
components are effective. Empirical analysis further investigates
the validity of the proposed method, and shows its benefit on complicated
problems. The source code can be found at {\color{CarnationPink}\href{https://github.com/AndyWu101/MAMC}{https://github.com/AndyWu101/MAMC}}.%
\end{abstract}
\renewcommand{\thefootnote}{}
\footnotetext{The authors are from department of Computer Science and Information Engineering, Fu Jen Catholic University, New Taipei City 242062, Taiwan.}
\renewcommand{\thefootnote}{\arabic{footnote}}

\section{Introduction}

Reinforcement learning (RL) has been studied for decades that is proved
powerful when dealing with problems and applications which is assumed
or is able to be formulated as a Markov decision process \cite{MDP_Puterman1990}.
Numerous applications have been successfully solved by RL methods
such as playing board games \cite{AlphaZero_silver2017}, training
large-language model \cite{LLMtraining_Neurips2022}, and controlling
humanoid \cite{DRLrobots_AAAI2025}. RL methods are of several types,
including value-based approach, policy gradient approach, policy optimization
approach, and actor-critic approach \cite{sutton2018}. This study
focus on actor-critic based RL methods due to its nice performance
on continuous control problems.

A common issue in RL is the huge or infinite space of states/actions,
making conventional tabular methods inapplicable, and a straight forward
solution is to construct approximation function for space transformation.
As the rapid growth in high performance computing and deep learning
\cite{DLbook_goodfellow2016}, leveraging deep learning for building
mapping function in RL methods, forming deep reinforcement learning
(DRL), becomes possible. One representative method of DRL is the deep
Q-learning (DQN) \cite{DQN_Mnih2015}, which adopted deep convolution
neural network to estimate the state-action function (a.k.a. $Q$
function).

Advanced issues in deep reinforcement learning have been studied and
investigated in past years \cite{DRLsurvey_AAAI2018}. Essential issues
covers learning stability \cite{AvgDQN_ICML2017,DDPG_ICLR2016,TD3_ICML2018},
estimation accuracy for handling issues of overestimation \cite{GPL_AAAI2023,OVD-Explorer_AAAI2024},
underestimation \cite{OAC_NEURIPS2019,BOO_NEURIPS2022} or both \cite{MaxMixMinQL_Abliz2024,MaxMinQ_ICLR2020},
sampling efficiency \cite{SAC_ICML2018,MEEE_ICRA2021,SMR_InformationSciences2024,WPVOP_TNNLS2024},
ensemble learning \cite{DARC_AAAI2022,REDQ_ICLR2021,SUNRISE_ICML2021,TQC_ICML2020,QMD3_AAAI2022}
and so forth, and hybridization of components for addressing these
issues is proved to gain effectiveness and learning efficiency \cite{Rainbow_AAAI2018}.
It is worth noting that these issues are highly correlated so that
ensemble learning could handle estimation accuracy, which may bring
learning stability and sampling efficiency, and thus results in better
performance and convergence.

This study proposes a novel method: multiple-actors-multiple-critics
(MAMC) deep deterministic reinforcement learning to address the above
issues. The main features of the MAMC are threefold: 1) The MAMC manipulates
multiple actors and critics in a concurrent manner without predetermined
relations, 2) The MAMC evaluates actors and critics as per a quantile-based
ensemble strategy, and 3) The MAMC selects actors for exploration
in learning on the basis of non-dominated sorting with respect to
skill and creativity factors. The emerging MAMC is capable of facilitating
nice exploration among multiple actors in the meantime improving and
smoothing the learning of critics, which is key to stabilize the guiding
force to actors.

The main contributions are listed as follows:
\begin{itemize}
\item Devise a parametric quantile-based ensemble estimator considering
multiple actors and multiple critics for the target values of critics
learning
\item Design an actor evaluation and selection approach based on skill and
creativity factors for exploration and exploitation
\item Theoretically prove the MAMC has stable learning and bounded estimation
bias
\item Empirically examine the quality and validity of the MAMC, and investigate
the run-time behavior of MAMC by inspecting into the proposed components
\end{itemize}
The rests of this study are organized as follows. Section \ref{sec:Related-Work}
reviews recent RL methods under actor-critic architectures, and Section
\ref{sec:Preliminaries} introduces preliminaries of this study. Sections
\ref{sec:Methodology} and \ref{sec:Theoretical-Analysis} in turn
gives details and theoretical analysis for the proposed method. Section
\ref{sec:Experimental-Results} examines the effectiveness for the
proposed method. Section \ref{sec:Conclusions} draws conclusions.

\section{Related Work\label{sec:Related-Work}}

\begin{table}
\caption{\label{tab:AC-review}A compilation of some recent proposed actor-critic
architectures according to the number of actors and critics}

\centering{}%
\begin{tabular}{llrrr}
\toprule 
 &  & \multicolumn{3}{c}{\#Critics}\tabularnewline
Method &  & Single & Double & Multiple\tabularnewline
\cmidrule(r){1-2}\cmidrule(lr){3-3}\cmidrule(lr){4-4}\cmidrule(l){5-5}\#Actors & Single & DDPG\cite{DDPG_ICLR2016} & %
TD3\cite{TD3_ICML2018}, SAC\cite{SAC_ICML2018}, & %
REDQ\cite{REDQ_ICLR2021}, MD3\cite{QMD3_AAAI2022},\tabularnewline
 &  &  & OAC\cite{OAC_NEURIPS2019}. & QWPVOP\cite{WPVOP_TNNLS2024}.\tabularnewline
\cmidrule(r){2-2}\cmidrule(lr){3-3}\cmidrule(lr){4-4}\cmidrule(l){5-5} & Double & - & DARC\cite{DARC_AAAI2022}. & -\tabularnewline
\cmidrule(r){2-2}\cmidrule(lr){3-3}\cmidrule(lr){4-4}\cmidrule(l){5-5} & Multiple & - & - & SUNRISE\cite{SUNRISE_ICML2021}.\tabularnewline
\bottomrule
\end{tabular}
\end{table}

The actor-critic architecture is proposed by Konda and Tsitsiklis
\cite{AC_NIPS1999}. Table \ref{tab:AC-review} compiles six out nine
categories of actor-critic architectures in terms of the number of
actors and critics for some recent proposed actor-critic-based RL
methods. To the best of our knowledge, it is merely no study for actor-critic
architectures with fewer number of actors than of critics.

\textbf{SASC.} For single-actor single-critic (SASC) architecture,
a representative study is the deep deterministic policy gradient (DDPG)
\cite{DDPG_ICLR2016}. DDPG ameliorated the learning stability and
efficiency of deep Q-network (DQN) by combining deep learning with
policy gradient for solving control problems with continuous action
space.

\textbf{SADC.} Beyond SASC, lots of methods are proposed with a single
actor and double critics, noted as SADC, for solving the issues of
overestimation and exploration. %
In \cite{TD3_ICML2018}, a twin delayed deep deterministic policy
gradient (TD3) was proposed. TD3 improved DDPG by adopting two critic
networks, where a minimum of the corresponded two target networks
are served as the computational basis of target value. TD3 also proposed
the delayed update of actor, i.e., a lower update frequency than critics,
for stabilizing the learning of the actor. Soft Actor-Critic (SAC)
considered stochastic policy and introduced soft value function for
training the two critics of soft Q-function \cite{SAC_ICML2018}.
Specifically, SAC trained a stochastic policy network to transform
noise to an action for a given state as condition, and the training
depends on a policy gradient for maximizing the randomness of the
resulting actions, and the approximated Q values obtained from the
minimum of the two critics as TD3. Different from TD3, SAC trained
the two critics independently according to the target soft value function
network, which is soft-updated by the soft value function network,
while the soft value function network is trained by minimizing the
different to the target value, which is calculated as the expectation
of state-action value of the minimum of the two critics over action
given by the actor. Optimistic Actor-Critic (OAC) further pointed
out the issues of inefficient exploration owing to insufficient pessimistic
in TD3 and SAC, and proposed an amelioration to guide the exploration
according to the approximated lower and upper bounds of the state-action
value function \cite{OAC_NEURIPS2019}.

\textbf{SAMC.} From the observation of improvement from SASC to SADC,
many methods considered increasing the number of critics for improving
the estimation accuracy, forming the single-actor multi-critic architecture
(SAMC). %
Randomized ensembled double Q-learning (REDQ) \cite{REDQ_ICLR2021}
estimated the state-action value using the same strategy of minimum
the same as TD3, yet the two critics were randomly selected from a
pool of critics. REDQ also introduced a high update-to-date (UTD)
ratio of 20 to address the issue of sample efficiency. For addressing
the estimation accuracy issue, quasi-median Q-learning (QMQ) used
the quasi-median among multiple state-action values, each of which
from a critic, to estimate the state-action value, and applied on
TD3, forming the QMD3. The QMD3 trained actor with delay the same
as TD3, but each update is guided by all critics rather than a single
one for exploration improvement. Weakly pessimistic value estimation
and optimistic policy optimization (WPVOP) \cite{WPVOP_TNNLS2024}
proposed weakly pessimistic value estimation and optimistic policy
optimization; the former increased and smoothed the lower confidence
bound, whilst the latter encourages and increases the state-action
value, as the maximum action of minimum state-action values, if the
distribution of state-action values with different actions on a given
state is centralized, i.e., the standard deviation less than some
threshold.

\textbf{DADC.} From single actor to double actors, double actors and
regularized critics (DARC) \cite{DARC_AAAI2022} adopted double actors
as well as double critics (DADC) and proposed soft target value as
a linear combination of the minimum and maximum state-action values
of the two actions given by two target actors, each of which is a
minimum over two target critics. DARC revised the loss function adopted
in TD3 by introducing a weighted regularization term of cross-critic
error, i.e., the difference between the two critics.

\textbf{MAMC.} For multi-actor multi-critic (MAMC) architecture, an
early MAMC method is the Simple UNified framework for ReInforcement
learning using enSEmbles (SUNRISE) \cite{SUNRISE_ICML2021}. SUNRISE
manipulated multiple SAC agents, each contained a pair of soft Q-function
and an actor. SUNRISE integrated weighted Bellman backup, which decreases
the influence from high variance transitions, and upper confidence
bound (UCB) exploration \cite{UCBQ_2017archive}.

\section{Preliminaries\label{sec:Preliminaries}}

Given a Markov decision process $\begin{gathered}(\mathcal{S},\text{\ensuremath{\mathcal{A}},}\mathcal{P},\mathcal{R},\gamma)\end{gathered}
$ with state space $\mathcal{S}$, action space $\mathcal{A}$, a state
transition probability $\mathcal{P}_{s,s^{\prime}}^{a}$, a reward
function $\mathcal{R}_{s,a}=\mathbb{E}[R_{t+1}|S_{t}=s,A_{t}=a]$,
and a discount factor $\gamma$, reinforcement learning aims at learning
policy $\begin{gathered}\pi\end{gathered}
$ to achieve optimal return from rewards. A famous method is the Q-learning
\cite{Watkins1992}, which learns a state-action value function for
estimating the reward function $\begin{gathered}\mathcal{R}_{s,a}\end{gathered}
$
\begin{gather}
\begin{aligned}Q^{\pi}(s,a) & =\mathbb{E}[R_{t+1}|S_{t}=s,A_{t}=a]\\
 & =\mathbb{E}[r_{t+1}+\gamma Q^{\pi}(S_{t+1}=s^{\prime},A_{t+1}=\pi(s_{t+1}))|S_{t}=s,A_{t}=a]
\end{aligned}
\,.\label{eq:state-action-value-function}
\end{gather}
 The estimation forms a Bellman equation, which can be solved by temporal
difference (TD) \cite{TD_Sutton1988,TD_Tesauro1995} methods. TD methods
approximate the expected return by gradually lowering down the TD
error, i.e., the difference of returns between the state-action value
$\begin{gathered}Q(s,a)\end{gathered}
$ and the TD-target $\begin{gathered}r_{t+1}+\gamma V(s_{t+1})\end{gathered}
$, where $\begin{gathered}V(s_{t+1})\end{gathered}
$ is the state-value function satisfying $\begin{gathered}V(s_{t+1})=Q(s_{t+1},\pi(s_{t+1}))\end{gathered}
$.%

Establishing approximation function to form a mapping from state space
to action space $\begin{gathered}\pi_{\phi}:\mathcal{S}\rightarrow\mathcal{A}\end{gathered}
$ and a mapping from state space and action space to a real-value $\begin{gathered}Q_{\theta}:\mathcal{S}\times\mathcal{A}\rightarrow\mathbb{R}\end{gathered}
$ by deep neural network forms deep reinforcement learning.%
{} According to \cite{TD3_ICML2018}, the update of critic then can
be made by minimizing the critic loss function:
\begin{equation}
J_{Q}(\theta)=\mathbb{E}_{(s,a,r,s^{\prime})\sim\mathcal{B}}[(Q_{\theta}(s,a)-\begin{gathered}r-\gamma V_{\phi}(s^{\prime};\theta^{\prime})\end{gathered}
)^{2}]\,,
\end{equation}
 subject to
\begin{equation}
V_{\phi}(s^{\prime};\theta^{\prime})=Q_{\theta^{\prime}}(s^{\prime},\pi_{\phi}(s^{\prime})+\epsilon)\,,
\end{equation}
 where $\theta^{\prime}$ is the parameters of critic target with
soft update, satisfying $\theta^{\prime}\gets\tau\theta+(1-\tau)\theta^{\prime}$,
and $\epsilon$ is the policy noise similar to the technique adopted
in SARSA learning \cite{sutton2018}. The soft update is for stabilizing
the learning of critic network using a fixed target. Then, the update
of actor is to minimize the actor loss function: 
\begin{equation}
J_{\pi}(\phi;\theta)=\mathbb{E}_{(s,a,r,s^{\prime})\sim\mathcal{B}}[-Q_{\theta}(s,\phi(s))]\,.\label{eq:objective-policy}
\end{equation}

\section{MAMC\label{sec:Methodology}}

\begin{table}
\centering{}\caption{\label{tab:Notation}Notation system}
\begin{tabular}{ll}
\toprule 
Symbol & Meaning\tabularnewline
\midrule 
$N_{A}$, $N_{C}$, $N_{\mathcal{B}}$ & Number of actors, critics, and mini-batch size\tabularnewline
$\pi_{\phi}$ & Actor network with parameter $\phi$\tabularnewline
$A$, $\tilde{A}$ & Actors and selected actors\tabularnewline
$C$, $C^{\prime}$ & Critics and target critics\tabularnewline
$Q_{\theta}$ ($Q_{\theta^{\prime}}$) & Critic (target) network with parameter $\theta$ ($\theta^{\prime}$)\tabularnewline
$\mathcal{R}$, $\mathcal{B}$ & Replay buffer and mini-batch\tabularnewline
$M$ & Sample multiple reuse\tabularnewline
$(s,a,r,s^{\prime})$ & Transition from state $s$ to next state $s^{\prime}$by action $a$
with reward $r$\tabularnewline
$\gamma$ & Discount factor\tabularnewline
$\vec{J}_{\mathrm{s}}(A)$, $\vec{J}_{\mathrm{c}}(A)$ & Skill and creativity factors of actors $A$\tabularnewline
$\prec$ & Crowded-comparison operator\tabularnewline
$\mathcal{N}(\mu,\sigma)$ & Gaussian distribution with mean $\mu$ and variance $\sigma^{2}$\tabularnewline
$\tau$ & Soft update ratio\tabularnewline
\bottomrule
\end{tabular}
\end{table}

This study proposes a multi-actor-multi-critic architecture-based
RL method: the Multi-Actor Multi-Critic deep deterministic reinforcement
learning (MAMC). There are three main features in the proposed MAMC,
including the adoption of multiple actors and critics without predefined
interaction, the quantile-based ensemble estimation, and the selection
of actors as per proposed skill and creativity factors for exploration
and exploitation. Table \ref{tab:Notation} provides the notation
system used in this study.

\subsection{The Overall Procedure}

\begin{algorithm}
\caption{\label{alg:MAMC}Main procedure of MAMC}

\begin{algorithmic}[1]

\State Initialize a set of $N_{A}$ actor networks $A$ with random
parameters $\left\{ \phi_{i}\right\} _{1\leq i\leq N_{A}}$

\State Initialize a set of $N_{C}$ critic networks $C$ with random
parameters $\left\{ \theta_{j}\right\} _{1\leq j\leq N_{C}}$

\State Initialize a set of $N_{C}$ target networks $C^{\prime}$
with critics $\theta_{j}^{\prime}\gets\theta_{j}$ for $1\leq j\leq N_{C}$

\State Initialize replay buffer $\mathcal{R}$

\State $o\gets1$\Comment{Order of critics}

\While{ Not Terminated }

\State $\triangleright$\textbf{ Critics Learning}

\State $\{\mathcal{B}_{j}\}_{1\leq j\leq N_{C}}\sim\mathcal{R}$
\Comment{Sample a mini-batch from replay buffer $R$ for each critic}

\For{ $m\gets1$ \textbf{to} $M$ }\Comment{Sample multiple reuse}

\State Update $\theta_{j}$ on $\mathcal{B}_{j}$ according to Eqs.~(\ref{eq:estimation-target})
and (\ref{eq:objective-ensemble-critic}) for $1\leq j\leq N_{C}$

\State Update $\theta_{j}^{\prime}$ by soft update for $1\leq j\leq N_{C}$

\EndFor

\State $\triangleright$\textbf{ Actors Learning}

\State $\{\mathcal{B}_{i}\}_{1\leq i\leq N_{A}}\sim\mathcal{R}$
\Comment{Sample a mini-batch from replay buffer $R$ for each actor}

\For{ $m\gets1$ \textbf{to} $M$ }\Comment{Sample multiple reuse}

\State Update $\phi_{i}$ by $\theta_{o}$ on $\mathcal{B}_{i}$
according to Eq.~(\ref{eq:objective-policy}) for all $1\leq i\leq N_{A}$

\State $o\gets(o\;\mathrm{mod}\;N_{C})+1$\Comment{Guided by each critic in turn}

\EndFor

\State $\triangleright$\textbf{ Exploration}

\State $\tilde{A}\gets$Selection $(\vec{J}_{\mathrm{s}}(A;C),\vec{J}_{\mathrm{c}}(A;C),\prec)$
\Comment{Crowded-comparison operator}

\State $(r,s^{\prime})\gets Env\left(s,a=\pi_{\phi\sim\tilde{A}}(s)+\epsilon\right),\ \epsilon\sim\mathcal{N}(0,\sigma)$
\Comment{Interact with environment}

\State $\mathcal{R}\gets\mathcal{R}\cup(s,a,r,s^{\prime})$

\State $\pi^{*}\gets\arg\max_{\phi}\:J_{\mathrm{s}}(\phi;C)$

\EndWhile

\State \Return $\pi^{*}$

\end{algorithmic}
\end{algorithm}

Algorithm \ref{alg:MAMC} gives the main procedure of the proposed
MAMC. At initialization, the MAMC generates a set of $N_{A}$ actor
networks $A$ and a set of $N_{C}$ critic networks $C$ with random
parameters, and set the parameters of each target network according
to the parameters of its corresponded critic network. The replay buffer
$\mathcal{R}$ is also initialized by random actions of a predefined
size. During each iteration, there are three main stages: critics
learning stage, actors learning state, and exploration stage.

\subsection{Quantile-based Ensemble Estimation}

In critics learning stage, $N_{C}$ sets of mini-batch $\{\mathcal{B}_{j}\}_{1\leq j\leq N_{C}}$
are sampled from the replay buffer $\mathcal{R}$, and each critic
is trained on a specific mini-batch for $M$ times for improving the
stability.

\begin{definition} For each transitions $(s,a,r,s^{\prime})\in\mathcal{B}_{j}$,
the TD-target for $j$th critic $Q_{\theta_{j}}$ is defined as the
median action of the $q$th-quantile among the critic targets:
\begin{equation}
y(s,a)=r+\gamma\hat{V}_{A}(s^{\prime})\,,
\end{equation}
 subject to
\begin{gather}
\begin{aligned}\hat{V}_{A}(s^{\prime};C^{\prime}) & =\mathrm{Med}(\{\hat{V}_{\phi_{i}}(s^{\prime};C^{\prime})\}_{1\leq i\leq N_{A}})\\
\hat{V}_{\phi_{i}}(s^{\prime};C^{\prime}) & =\mathrm{Quantile}_{q}(\{Q_{\theta_{j}^{\prime}}(s^{\prime},\pi_{\phi_{i}}(s^{\prime})+\epsilon)\}_{1\leq j\leq N_{C}})
\end{aligned}
\,.\label{eq:estimation-target}
\end{gather}

\end{definition}

The critic loss function is therefore defined as
\begin{equation}
J_{Q}(\theta_{j};C^{\prime})=\mathbb{E}_{(s,a,r,s^{\prime})\sim\mathcal{B}}[(Q_{\theta_{j}}(s,a)-\begin{gathered}r-\gamma V_{A}(s^{\prime};C^{\prime})\end{gathered}
)^{2}]\,.\label{eq:objective-ensemble-critic}
\end{equation}

All the target critics are soft-updated with parameter $\tau$ after
one out of $M$ iterations of training, which is capable of sharing
information to each target critic from all the other critic targets
and bring to the next iteration. 

For the learning of actors, the MAMC also sampled $N_{A}$ sets of
mini-batch $\{\mathcal{B}_{i}\}_{1\leq i\leq N_{A}}$ from the replay
buffer $\mathcal{R}$ as it does in critics learning stage. The training
of each actor $\pi_{\phi_{i}}$ is in turn guided by each critic $Q_{\theta_{j}}$
with objective $J_{\pi}(\phi_{i};\theta_{j})$ (cf.~Eq.~(\ref{eq:objective-policy}))
on its mini-batch $\mathcal{B}_{i}$. The idea of updating $M$ times
within a mini-batch for each actor and critics is similar to sample
multiple reuse (SMR) proposed in \cite{SMR_InformationSciences2024},
which is able to stabilize the learning sequence.

\subsection{Actor Evaluation, Exploration, and Exploitation}

After training of actors and critics, the exploration stage is to
select appropriate actors for interacting with the environment. The
evaluation of an actor $\pi_{\phi_{i}}$ is based on two factors,
i.e., skill and creativity, both are determined by the ensemble estimation
of state value function.

\begin{definition} Ensemble estimation of state value function is
defined as the $q$th-quantile of state-action value function over
critics $C$:
\begin{equation}
\hat{V}_{\phi_{i}}(s^{(k)};C)=\mathrm{Quantile}_{q}(\{Q_{\theta_{j}}(s^{(k)},\pi_{\phi_{i}}(s^{(k)}))\}_{1\leq j\leq N_{C}})\,,\label{eq:ensemble-estimation-v}
\end{equation}
 where $s^{(k)}$ is the $k$th transition in a mini-batch. The consideration
of skill factor guarantees the quality of interaction, whilst the
consideration of creativity factor preserves the diversity of interaction.

\end{definition}

The skill factor evaluates the optimality of an actor through the
scoring ability on the ensemble estimation
\begin{equation}
J_{\mathrm{s}}(\phi_{i};C)=N_{\mathcal{B}}^{-1}\sum_{k=1}^{N_{\mathcal{B}}}\hat{V}_{\phi_{i}}(s^{(k)})\,,
\end{equation}
 while the creativity factor examines the diversity of an actor on
the critics through the closeness of each critic to the ensemble estimation
with respect to mean absolute error
\begin{equation}
J_{\mathrm{c}}(\phi_{i};C)=N_{\mathcal{B}}^{-1}N_{C}^{-1}\sum_{k=1}^{N_{\mathcal{B}}}\sum_{j=1}^{N_{C}}|Q_{\theta_{j}}(s^{(k)},\pi_{\phi_{i}}(s^{(k)}))-\hat{V}_{\phi_{i}}(s^{(k)})|\,.
\end{equation}
 Both factors are expectation over a mini-batch. Note that the two
factors depends on all the critics as rather than critic targets since
the actors are guided by critics. The selection of actors on the two
factor hinges upon the crowed-comparison operator \cite{NSGAII_Deb2002}
by considering the two factors as two objective values. The top-$\sqrt{N_{A}}$
actors $\tilde{A}$ are selected, which serves as the candidate actors
for interaction with the environment. Specifically, an actor is randomly
picked from the candidate actors $\tilde{A}$ for determining a single
step of interaction with the environment. The MAMC also records an
optimal policy with highest skill factor for exploitation at each
iteration; that is, the MAMC only returns a single actor for inference
due to the efficiency in terms of time and space complexity.

\section{Theoretical Analysis\label{sec:Theoretical-Analysis}}

This section gives some nice properties for the MAMC. First, the target
values obtained by multiple actors are more stable in terms of variance
than using a single actor.

\begin{theorem} The variance of target values obtained by multiple
actors are less than that using a single actor.
\begin{equation}
\mathbb{V}[\hat{V}_{A}(s^{\prime};C^{\prime})]\leq\mathbb{V}[\hat{V}_{\phi}(s^{\prime};C^{\prime})]
\end{equation}

\end{theorem}

Similarly, the target values obtained by multiple critics are more
stable than using a single critic.

\begin{theorem} The variance of target values obtained by multiple
critics are less than using a single critic.
\begin{equation}
\mathbb{V}[\hat{V}_{\phi}(s^{\prime};C^{\prime})]\leq\mathbb{V}[\hat{V}_{\phi}(s^{\prime};\theta^{\prime})]
\end{equation}

\end{theorem}

Thus, the learning stability of the MAMC, with lowest variance, is
greater than SAMC and SASC.

Further, this study investigate the property of estimation error,
which is a good metric for indicating the estimation accuracy \cite{DARC_AAAI2022}. 

\begin{definition} The estimation error of MAMC is defined as the
difference between expectation of estimate values and the expectation
of optimal policy $\pi$.
\begin{equation}
\mathcal{E}_{A,C}=\mathbb{E}[\hat{V}_{A}(s^{\prime};C)]-\mathbb{E}[\hat{V}_{\phi^{*}}(s^{\prime};C)]
\end{equation}

\end{definition}

Then the MAMC holds the following properties.

\begin{theorem} The estimation error of MAMC is between the estimation
error of multiple actors with minimum and maximum critics.
\begin{equation}
\mathcal{E}_{A,Q_{\theta_{\min}}}\leq\mathcal{E}_{A,C}\leq\mathcal{E}_{A,Q_{\theta_{\max}}}
\end{equation}

\end{theorem}

\begin{theorem} The estimation error of MAMC is between the estimation
error of multiple critics with minimum and maximum actors.
\begin{equation}
\mathcal{E}_{\pi_{\phi_{\min}},C}\leq\mathcal{E}_{A,C}\leq\mathcal{E}_{\pi_{\phi_{\max}},C}
\end{equation}

\end{theorem}

Hence, the estimation error of MAMC is in between the maximum and
minimum of SAMC and MASC. The proofs of above theorems will be given
in supplementary material Section \ref{sec:Proof} due to space limitation.

\section{Experimental Results\label{sec:Experimental-Results}}

This section examines the performance of the proposed MAMC method
in terms of effectiveness and efficiency through experiments. Further
analysis is made for showing the effectiveness of the proposed components,
the sensitivity of introduced hyperparameters, and the validity of
the MAMC.

\subsection{Experimental Settings\label{subsec:Experimental-Settings}}

The experiments are conducted on a set of five test environments chose
from the well-known MuJuCo benchmark \cite{MuJoCo}, including Hopper-v5,
HalfCheetah-v5, Walker2d-v5, Ant-v5, and Humanoid-v5. These environments
are all have continuous state and action spaces with different dimensions.
Regarding the dimensionality, the difficulty of each environment can
be regarded as either simple (Hopper-v5), medium (HalfCheetah-v5,
Walker2d-v5), or hard (Ant-v5, Humanoid-v5). The properties of these
environment can be found in the supplementary material Section \ref{subsec:MuJoco}.

This study selects four state-of-the-art RL methods for performance
comparison, including two with deterministic policy: TD3\cite{TD3_ICML2018}
and DARC\cite{DARC_AAAI2022}, and two with stochastic policy: SAC\cite{SAC_ICML2018},
REDQ\cite{REDQ_ICLR2021}. Both TD3 and SAC used a single actor with
two critics. REDQ also adopts a single actor but with ten critics,
while DARC exploits two actors and two critics. The proposed MAMC
utilizes ten actors and ten critics. An analysis on the number of
actors and critics can be found in the supplementary material Section
\ref{subsec:The-number-of-AC}. In addition, as the MAMC considers
sample multiple reuse (SMR) \cite{SMR_InformationSciences2024}, all
the four test methods are implemented as SMR versions, which are reported
with better performance than the original versions, for a fair comparison.

The hyperparameter settings for the four baseline methods follow their
original suggestions. The termination criterion is set to 300k environmental
steps. All experiments conducted 10 trials, and each trial is an average
over twenty seeds for return if not stated. All figures are uniformly
smoothed. For significance analysis, this study adopts the Wilcoxon
ranksum test with $.05$ significant level. The error bars are within
the range $[\mu-\sigma,\mu+\sigma]$, which are generated by standard
deviations with the assumption of normally distributed errors. For
more details about the experimental settings, please refer to the
supplementary material Section \ref{sec:Detailed-Experimental-Settings}.

\begin{table}
\caption{\label{tab:Wilcoxon-signed-rank-quality}Wilcoxon signed rank test
for TD3 and DARC compared with the MAMC at early (100k), middle (200k),
and late stage (300k). The win/tie/lose denotes the number of environments
that the MAMC is significantly superior, equal, and inferior to the
corresponding test method.}

\centering{}%
\begin{tabular}{lcccc}
\toprule 
Stage (win/tie/lose) & \multicolumn{1}{c}{TD3-SMR} & \multicolumn{1}{c}{DARC-SMR} & \multicolumn{1}{c}{SAC-SMR} & \multicolumn{1}{c}{REDQ-SMR}\tabularnewline
\midrule 
\multicolumn{1}{l}{100k} & 3/2/0 & 2/2/1 & 3/2/0 & 0/4/1\tabularnewline
\multicolumn{1}{l}{200k} & 2/3/0 & 1/4/0 & 2/2/1 & 0/4/1\tabularnewline
\multicolumn{1}{l}{300k} & 2/3/0 & 1/3/1 & 1/4/0 & 0/3/2\tabularnewline
\bottomrule
\end{tabular}
\end{table}

\begin{figure}
\begin{centering}
\includegraphics[width=0.96\columnwidth]{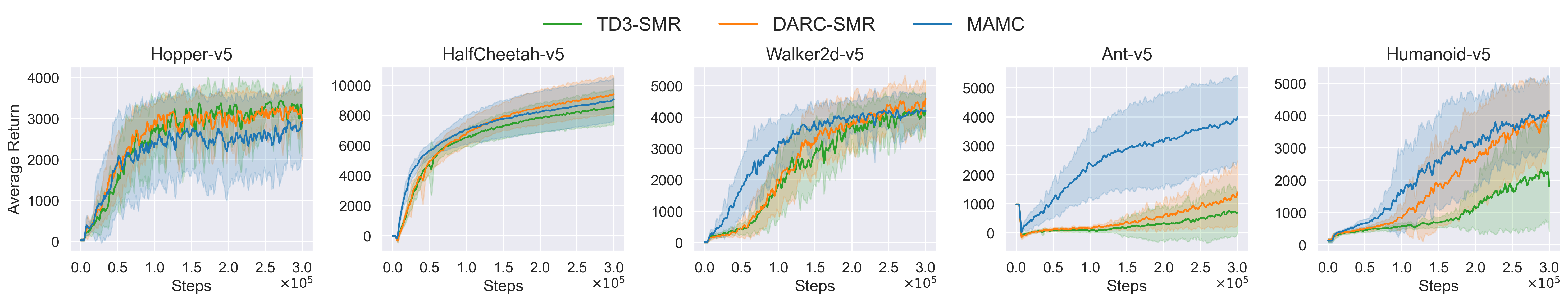}
\par\end{centering}
\centering{}\includegraphics[width=0.96\columnwidth]{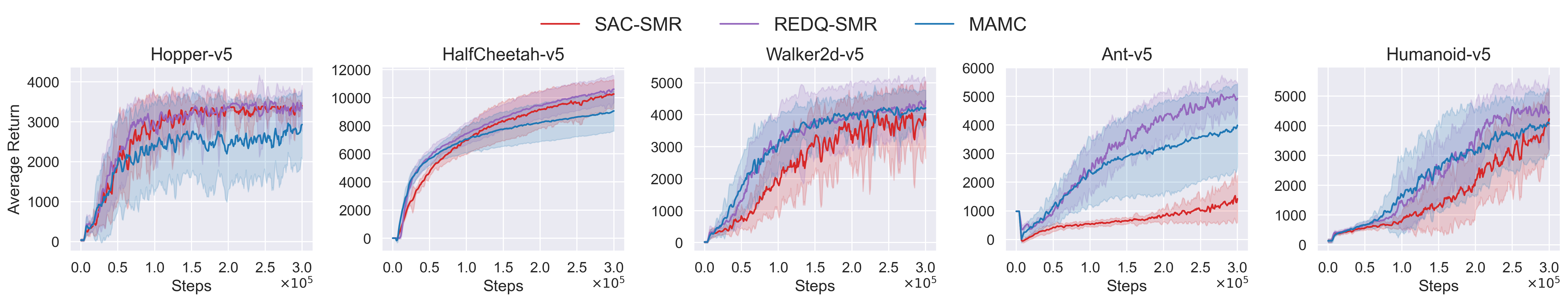}\caption{\label{fig:Average-return-TD3-SAC}Average return against environment
steps for TD3-based and SAC-based methods by comparison with the MAMC
on the five environments}
\end{figure}

\subsection{Effectiveness}

Table \ref{tab:Wilcoxon-signed-rank-quality} compares the Wilcoxon
signed rank test for TD3 and DARC compared with the MAMC at early
(100k), middle (200k), and late stage (300k). The details are provided
in supplementary material Section \ref{subsec:Mean-and-stdandard}.
At early stage, the MAMC achieves better quality than the two deterministic
methods TD3-SMR and DARC-SMR. Comparing to the two SAC-based methods,
the MAMC outperforms SAC-SMR but performs slightly worse than REDQ-SMR
on the Hopper-v5 environment. At middle stage, the MAMC still betters
TD3-SMR and DARC-SMR, yet the improvement becomes smaller than that
at early stage. As for the two SAC-based methods, the trend on REDQ-SMR
keeps, while the improvement on SAC-SMR also decreases. At late stage,
the lead to TD3-SMR, DARC-SMR, and SAC-SMR further shrinks that the
MAMC is slightly superior to TD3-SMR and SAC-SMR, but is comparable
to DARC-SMR. The REDQ-SMR further surpasses the MAMC on Humanoid-v5
environment. These results reflect the merits of MAMC at early and
middle stage, and the demerit at late stage.

\subsection{Efficiency}

Figure \ref{fig:Average-return-TD3-SAC} draws the average return
against environment steps for TD3-based and SAC-based methods by comparison
with the MAMC on the five environments. Compared with TD3-based methods,
the MAMC gains faster convergence on the three more complicated environments,
i.e., Walker2d-v5, Ant-v5, and Humanoid-v5. Similarly, the MAMC converges
faster than SAC-SMR on these three environments, yet the REDQ-SMR
converges nicer than the MAMC on all except Walker2d-v5. These results
validate the efficiency of the MAMC against the two deterministic
method TD3-SMR and DARC-SMR, and the simpler stochastic method SAC-SMR.

\subsection{Components Analysis}

\begin{table}
\caption{\label{tab:mean-std-MO-SO}Average and standard deviation of return
for the MAMC with single-objective and multi-objective actor selection
strategies on Ant-v5 over eight trials at early (100k), middle (200k),
and late stage (300k). The bold symbol implies the highest value.}

\centering{}%
\begin{tabular}{lrrrrrr}
\toprule 
Stage & \multicolumn{2}{c}{100k} & \multicolumn{2}{c}{200k} & \multicolumn{2}{c}{300k}\tabularnewline
\cmidrule(lr){2-3}\cmidrule(lr){4-5}\cmidrule(l){6-7}Ant-v5 & SO & MO & SO & MO & SO & MO\tabularnewline
\midrule 
Mean & 1980 & \textbf{2701} & 2805 & \textbf{3611} & 3395 & \textbf{4276}\tabularnewline
Std. & 800 & 1235 & 1014 & 1631 & 1278 & 1260\tabularnewline
\bottomrule
\end{tabular}
\end{table}

\begin{figure}
\begin{centering}
\includegraphics[width=0.2\columnwidth]{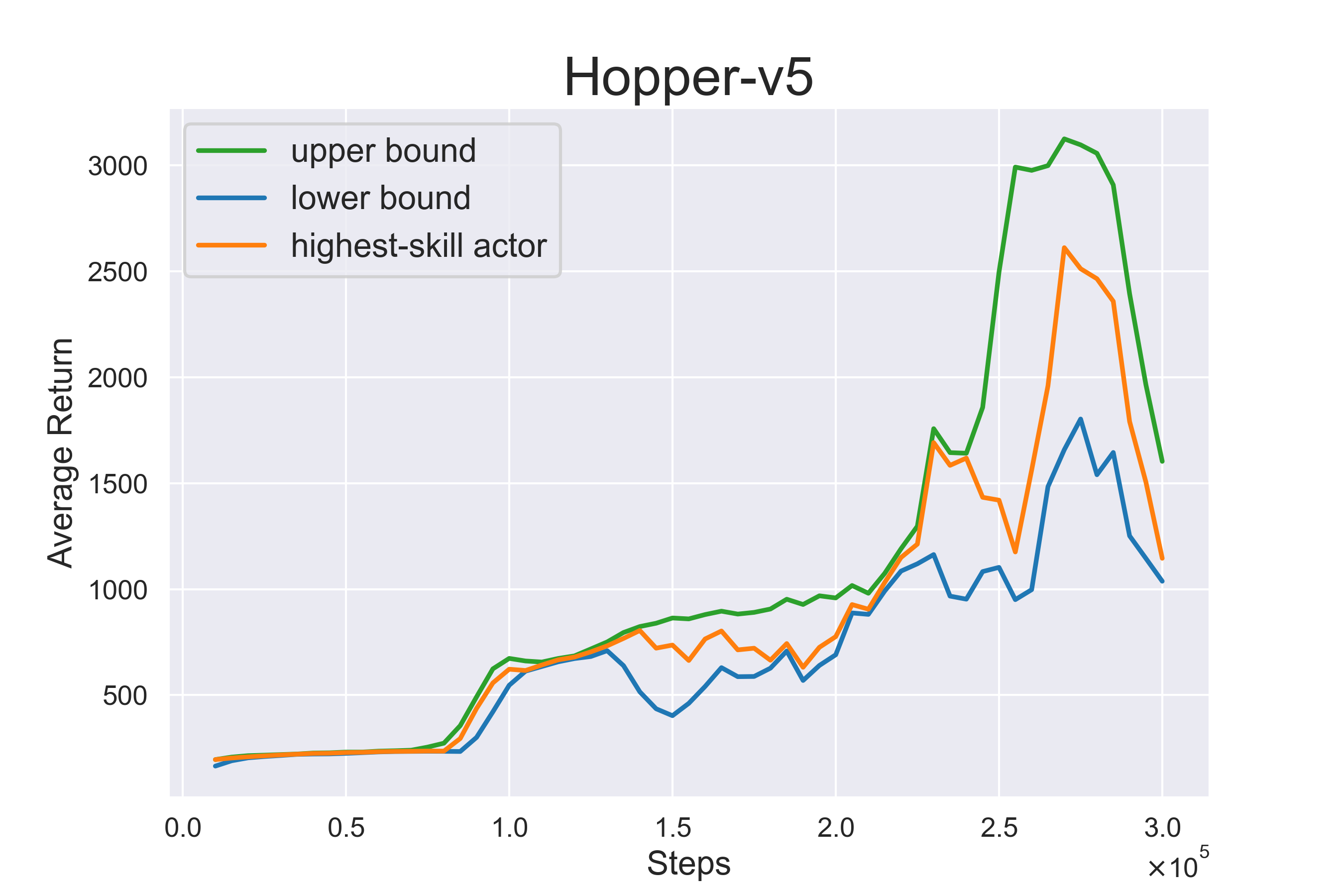}\includegraphics[width=0.2\columnwidth]{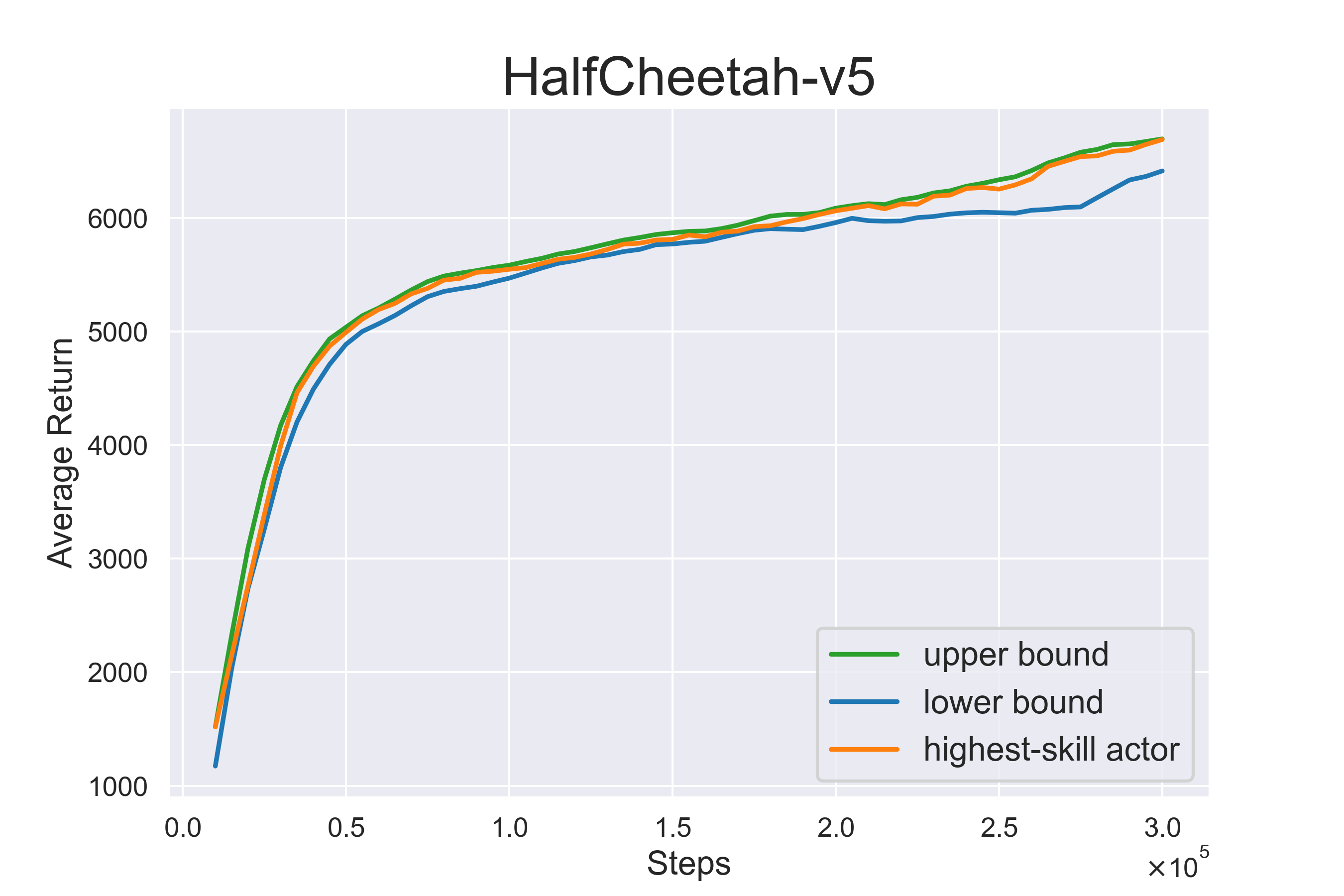}\includegraphics[width=0.2\columnwidth]{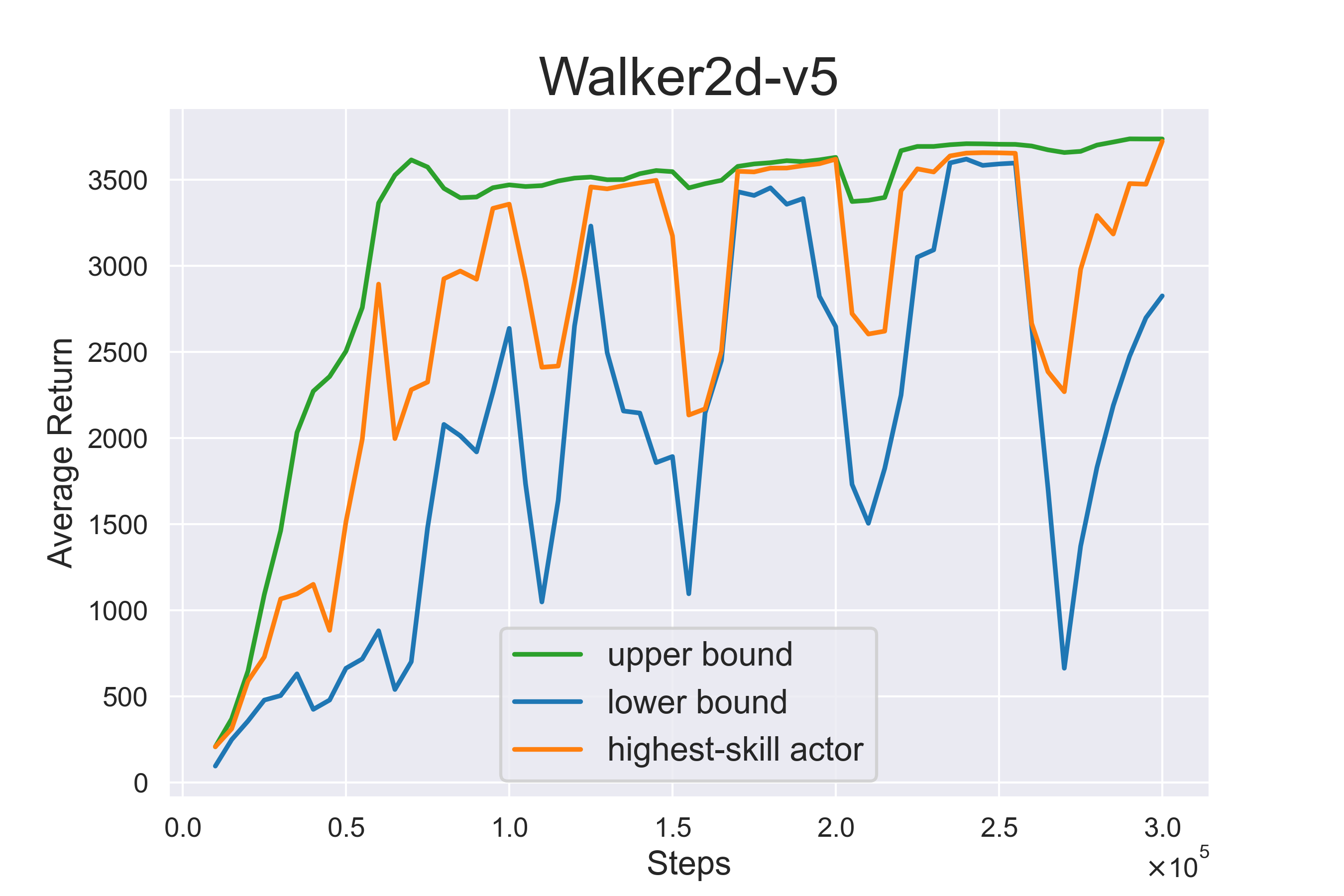}\includegraphics[width=0.2\columnwidth]{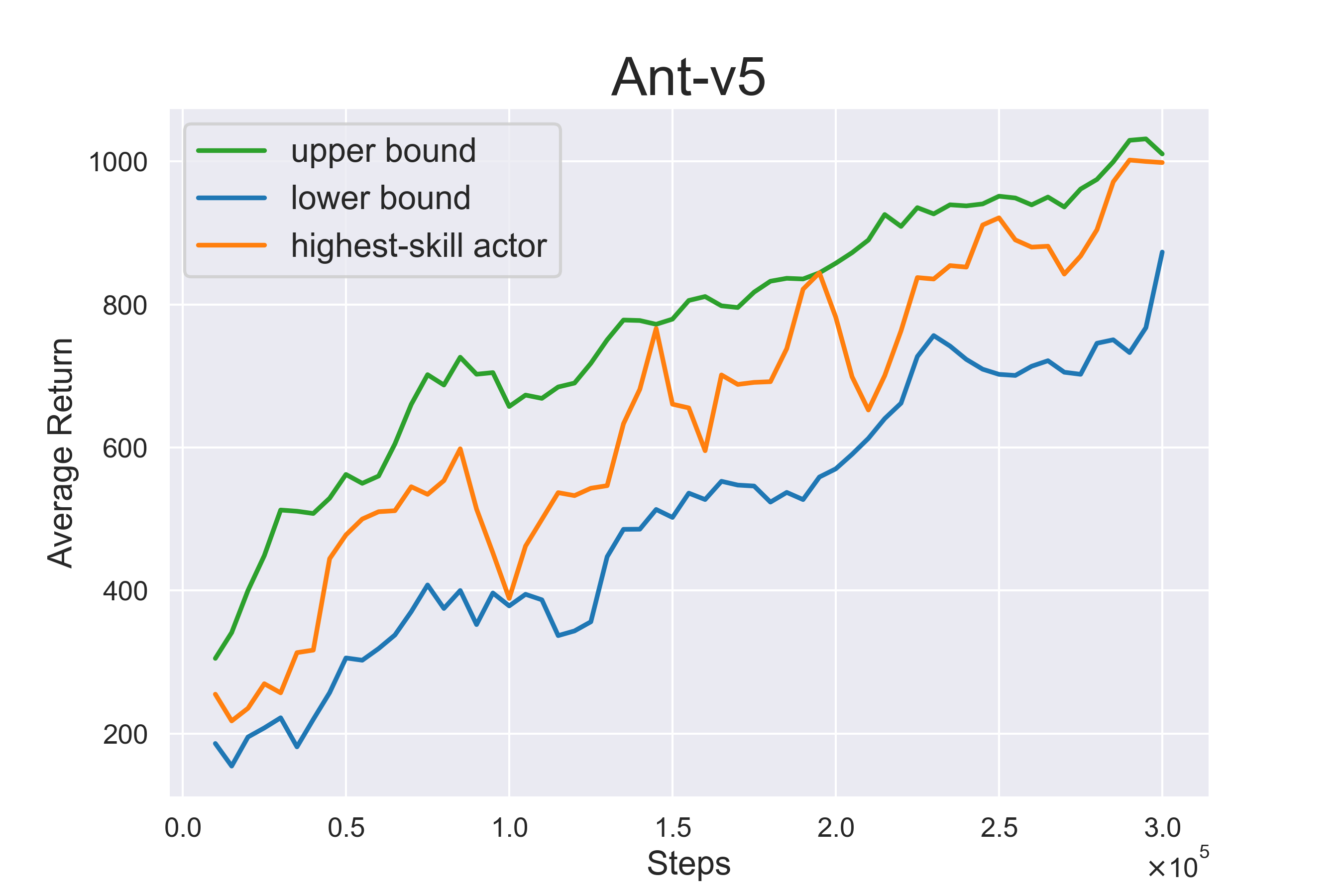}\includegraphics[width=0.2\columnwidth]{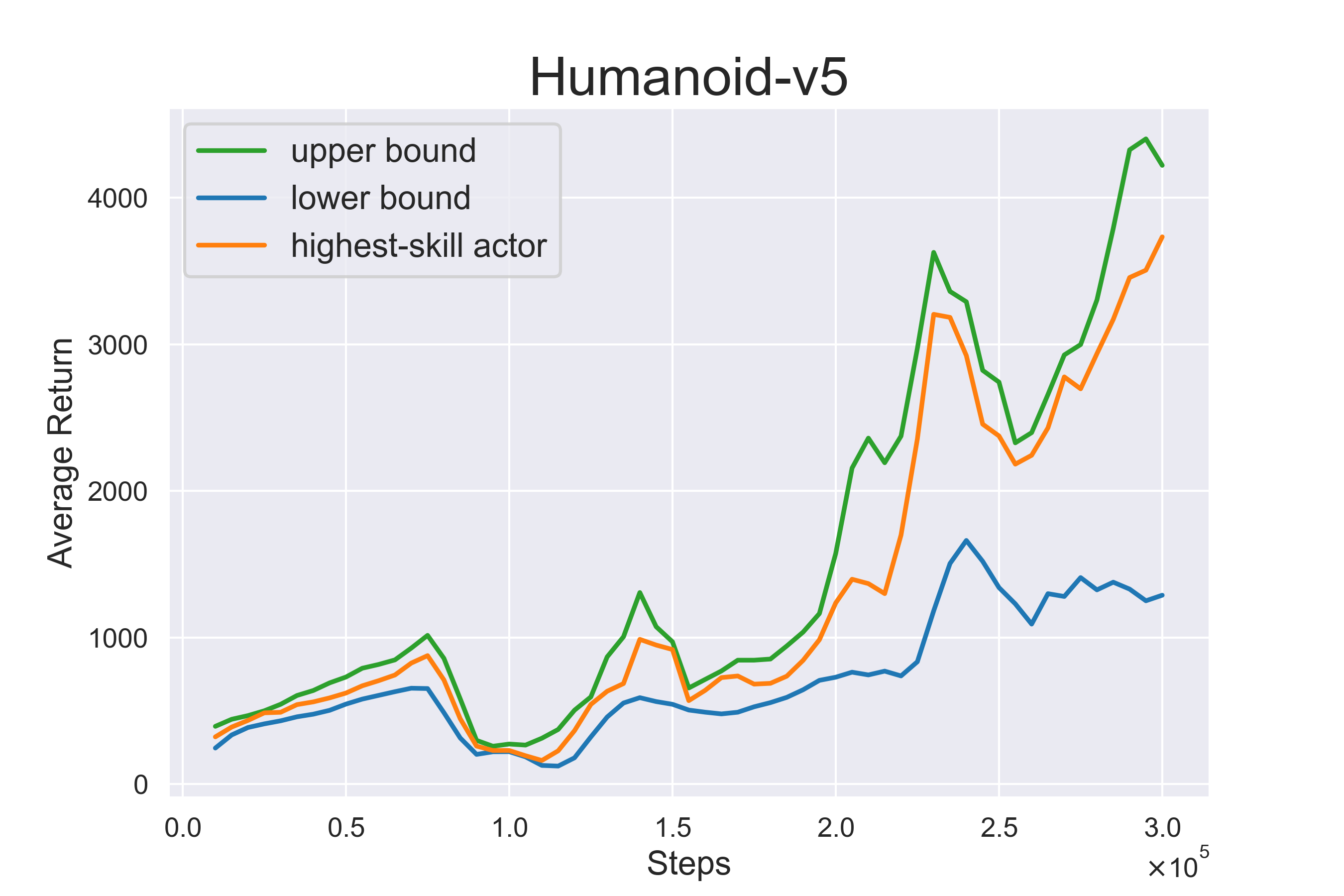}
\par\end{centering}
\centering{}\caption{\label{fig:AvgReturn-UB-LB}Average return for the best, worst, and
skilled (selected) actors in the MAMC in a specific trial on the five
test environments}
\end{figure}

Table \ref{tab:mean-std-MO-SO} lists the average and standard deviation
of return for the MAMC with single-objective (MAMC-SO) and multi-objective
actor selection strategies on Ant-v5 over eight trials. The MAMC-SO
averages the skill and creativity factors and selects the top actors
by sorting for exploration. The exploitation selection mechanism for
MAMC-SO and MAMC is the same. From the table, the MAMC performs better
than MAMC-SO at all the three stages, which verifies the effectiveness
of the proposed multi-objective actor selection mechanism.

Figure \ref{fig:AvgReturn-UB-LB} plots the average return for the
best (upper bound), worst (lower bound), and skilled (selected) actors
in the MAMC in a specific trial on the five test environments. On
HalfCheetah-v5 and Humanoid-v5, the MAMC is capable of selecting good
actor approaching the upper bound, to wit, the best actor. For Hopper-v5,
Walker2d-v5, and Ant-v5, the MAMC tracks the moving upper bound, and
in most of the time the selected actor having quality beyond the average
of upper and lower bounds. These results validate the effectiveness
of the proposed skill factor for actor selection for exploitation.%

\subsection{Sensitivity Analysis}

\begin{table}
\caption{\label{tab:mean-std-q}Average and standard deviation of return for
the MAMC with different quantile parameter $q$}

\centering{}%
\begin{tabular}{lrrrrr}
\toprule 
$q$ & \multicolumn{1}{c}{$=0.1$} & \multicolumn{1}{c}{$=0.2$} & \multicolumn{1}{c}{$=0.3$} & $=0.4$ & $=0.5$\tabularnewline
\midrule
HalfCheetah-v5 & 9119$\pm$1077 & 8153$\pm$1115 & 9117$\pm$1070 & 9191$\pm$1043 & \textbf{9466}$\pm$1256\tabularnewline
\multicolumn{1}{l}{Walker2d-v5} & 3188$\pm$1516 & \textbf{4324}$\pm$1038 & 4083$\pm$927 & 3385$\pm$1039 & 1406$\pm$561\tabularnewline
\bottomrule
\end{tabular}
\end{table}

\begin{figure}
\begin{centering}
\includegraphics[width=0.2\columnwidth]{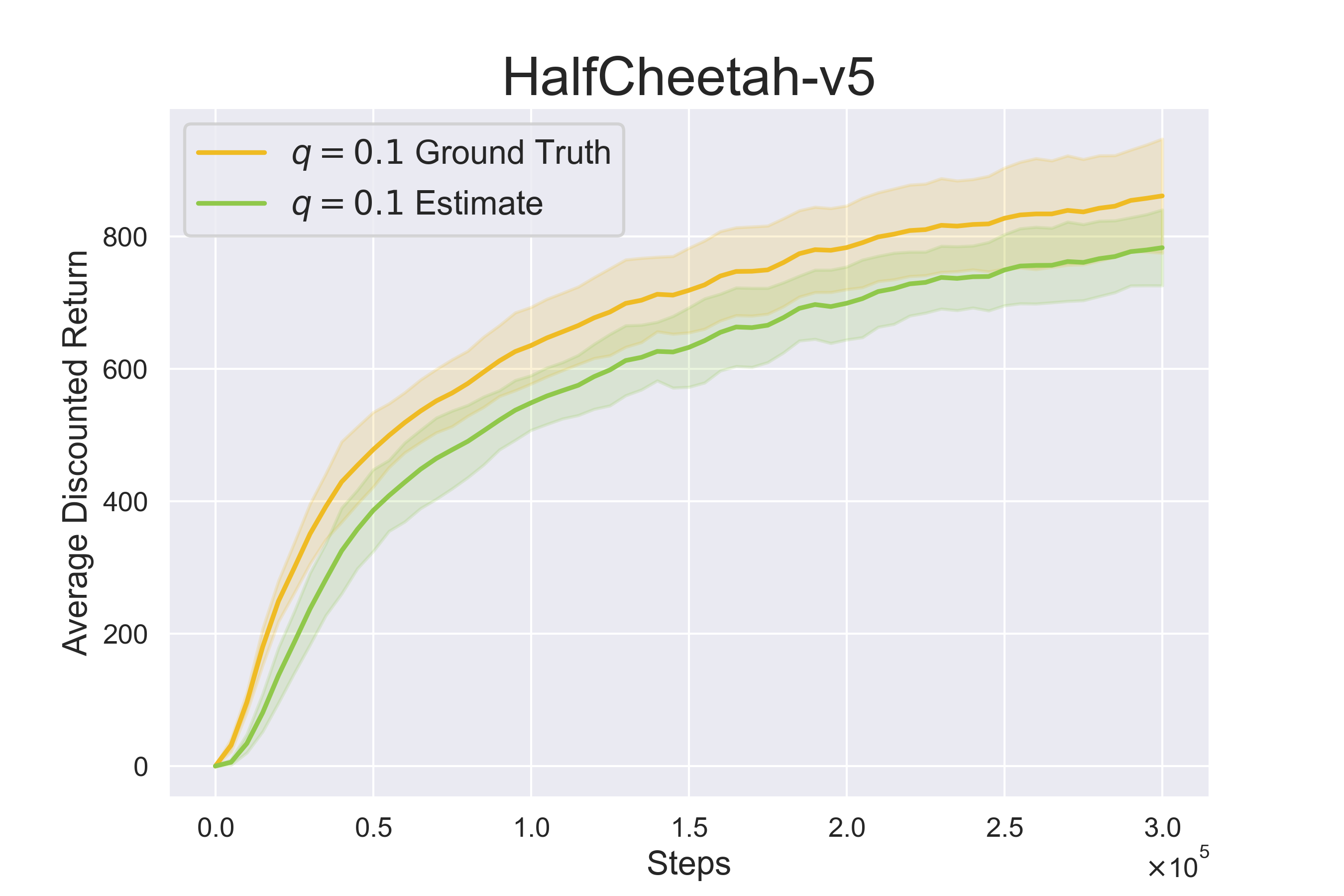}\includegraphics[width=0.2\columnwidth]{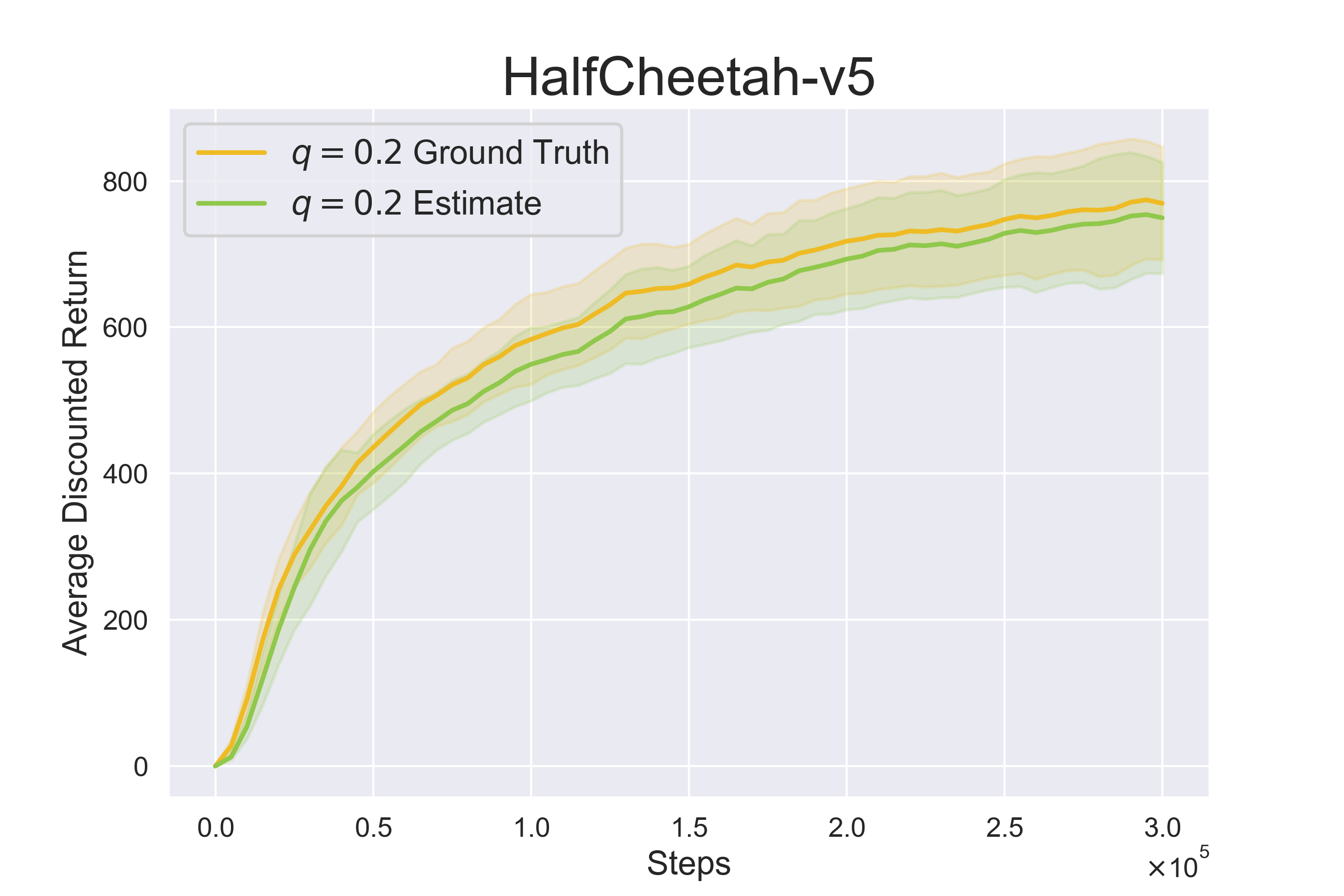}\includegraphics[width=0.2\columnwidth]{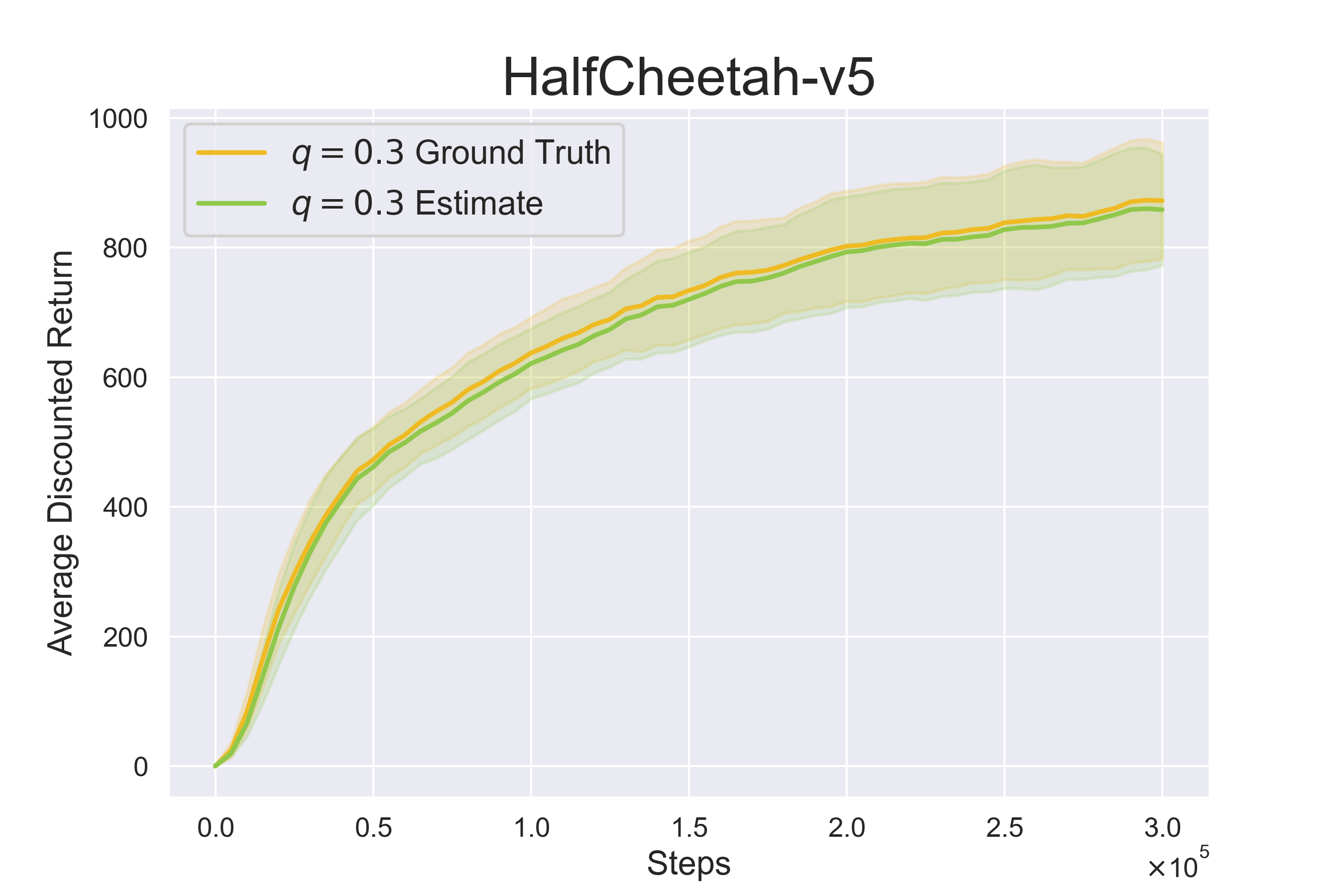}\includegraphics[width=0.2\columnwidth]{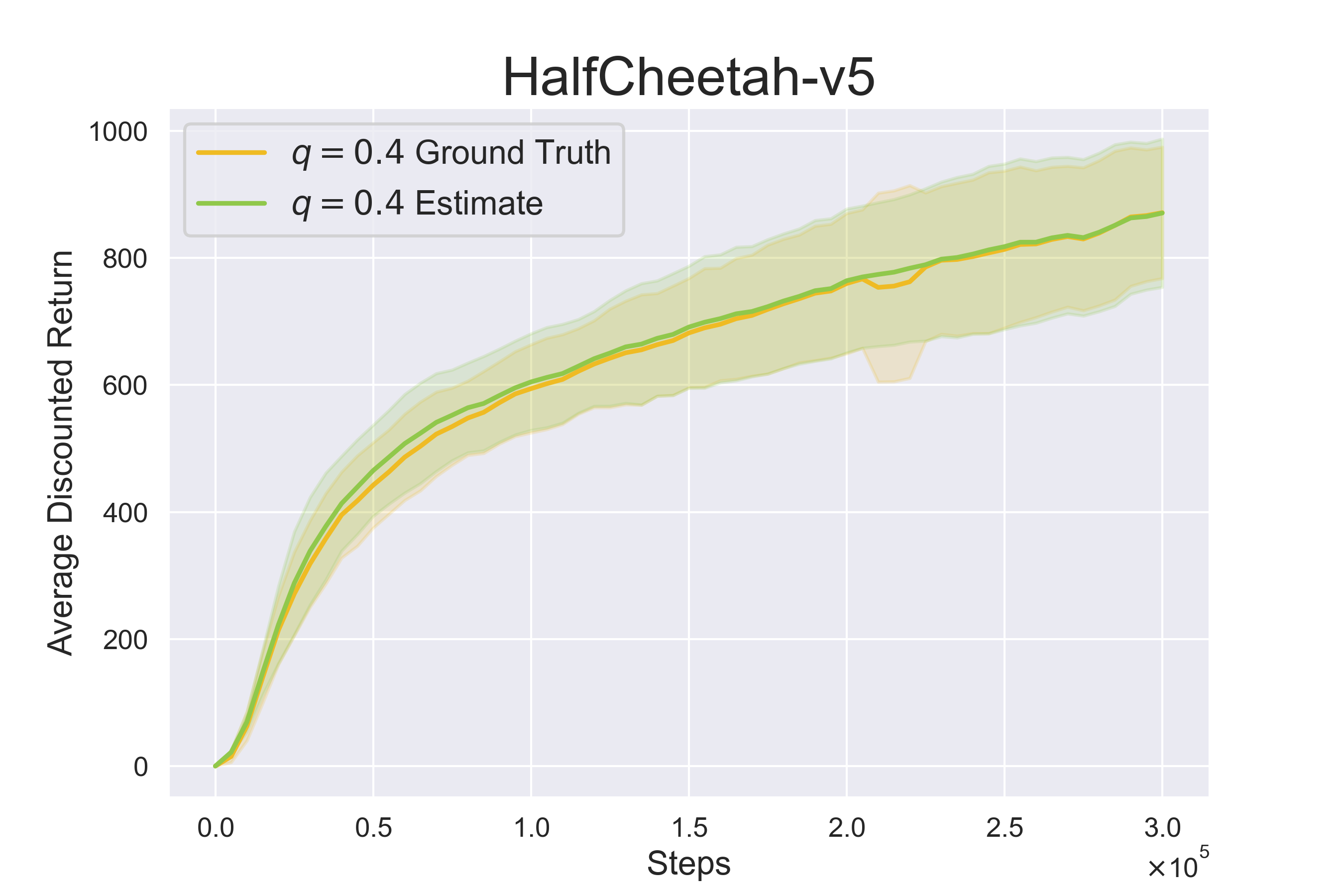}\includegraphics[width=0.2\columnwidth]{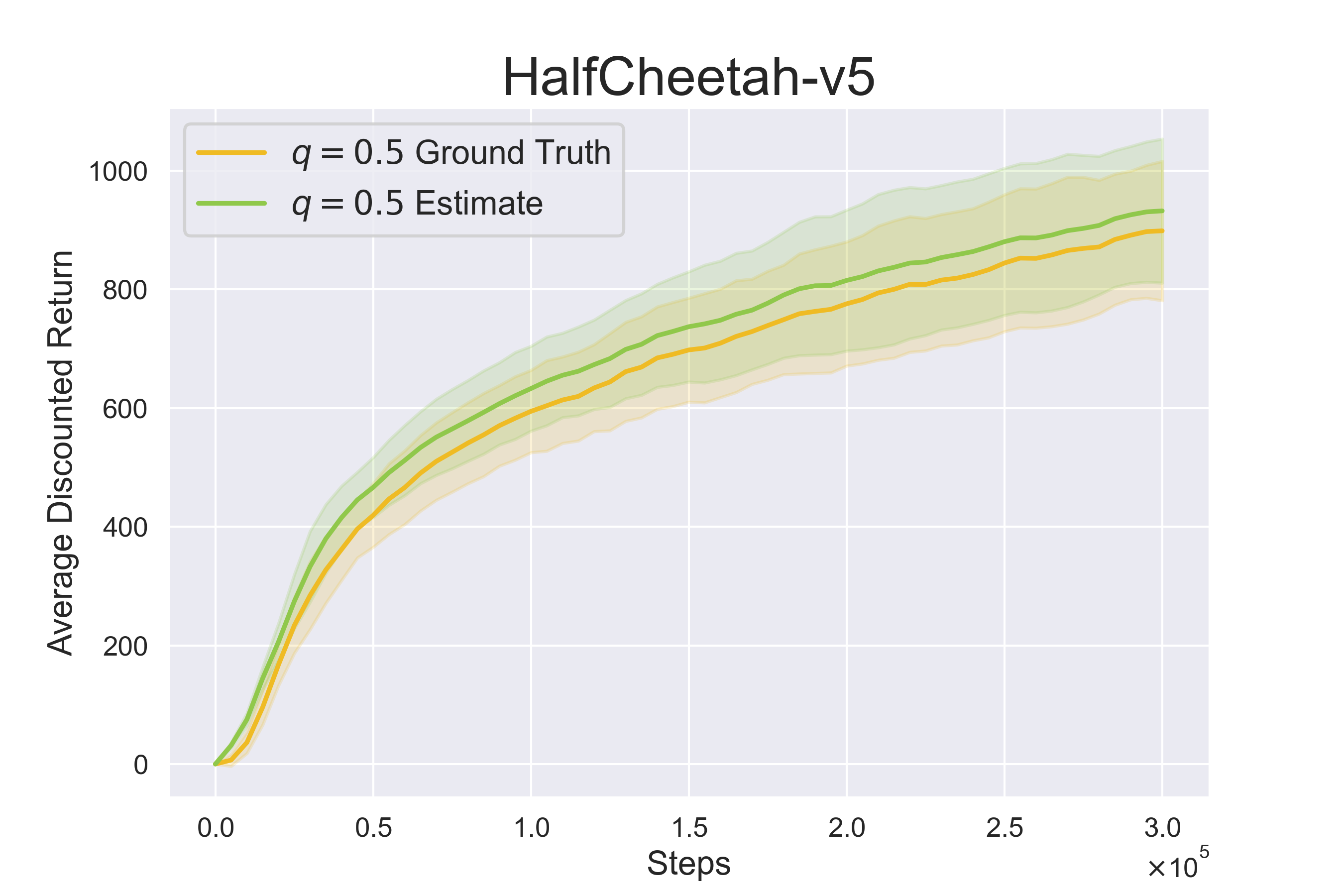}
\par\end{centering}
\begin{centering}
\includegraphics[width=0.2\columnwidth]{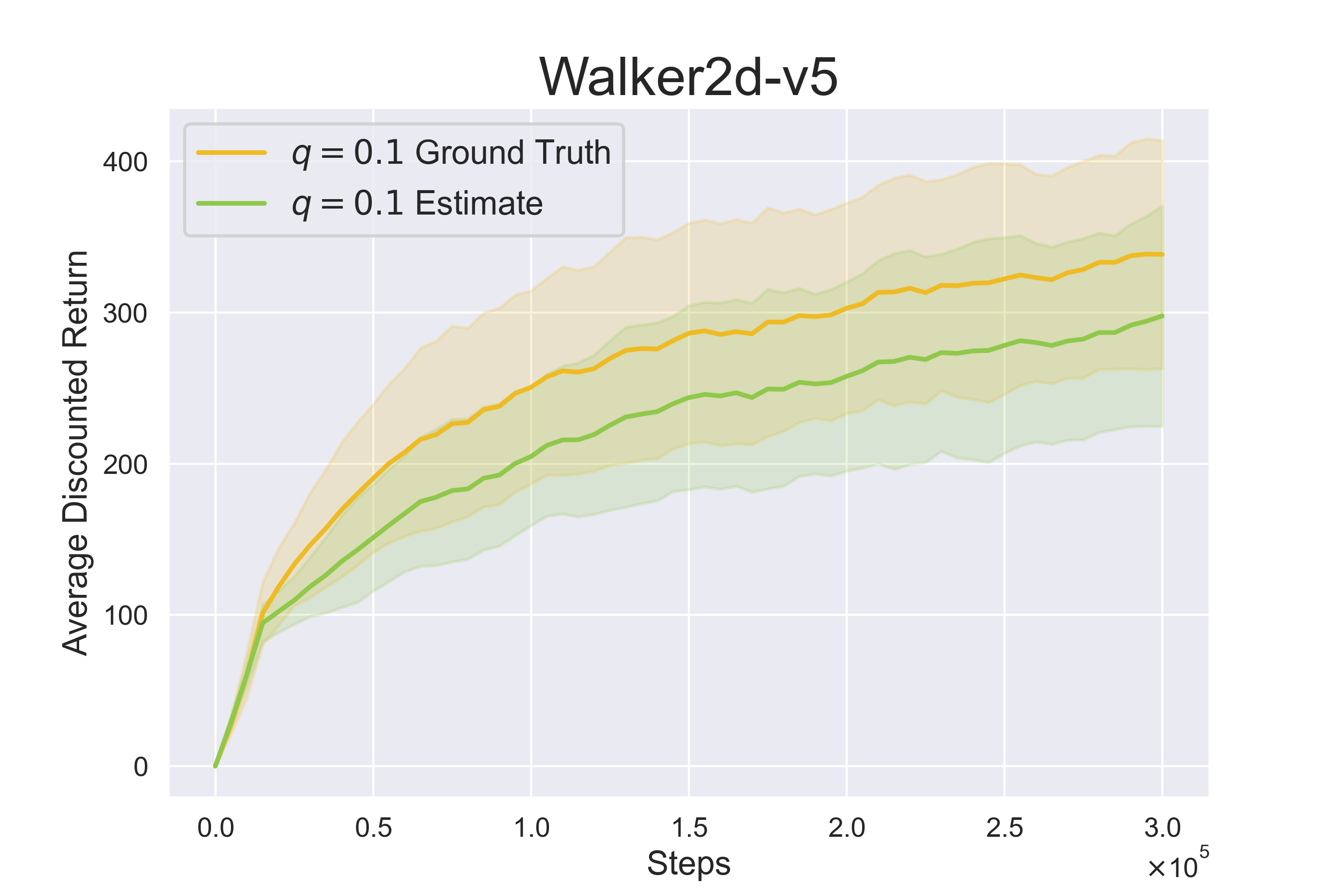}\includegraphics[width=0.2\columnwidth]{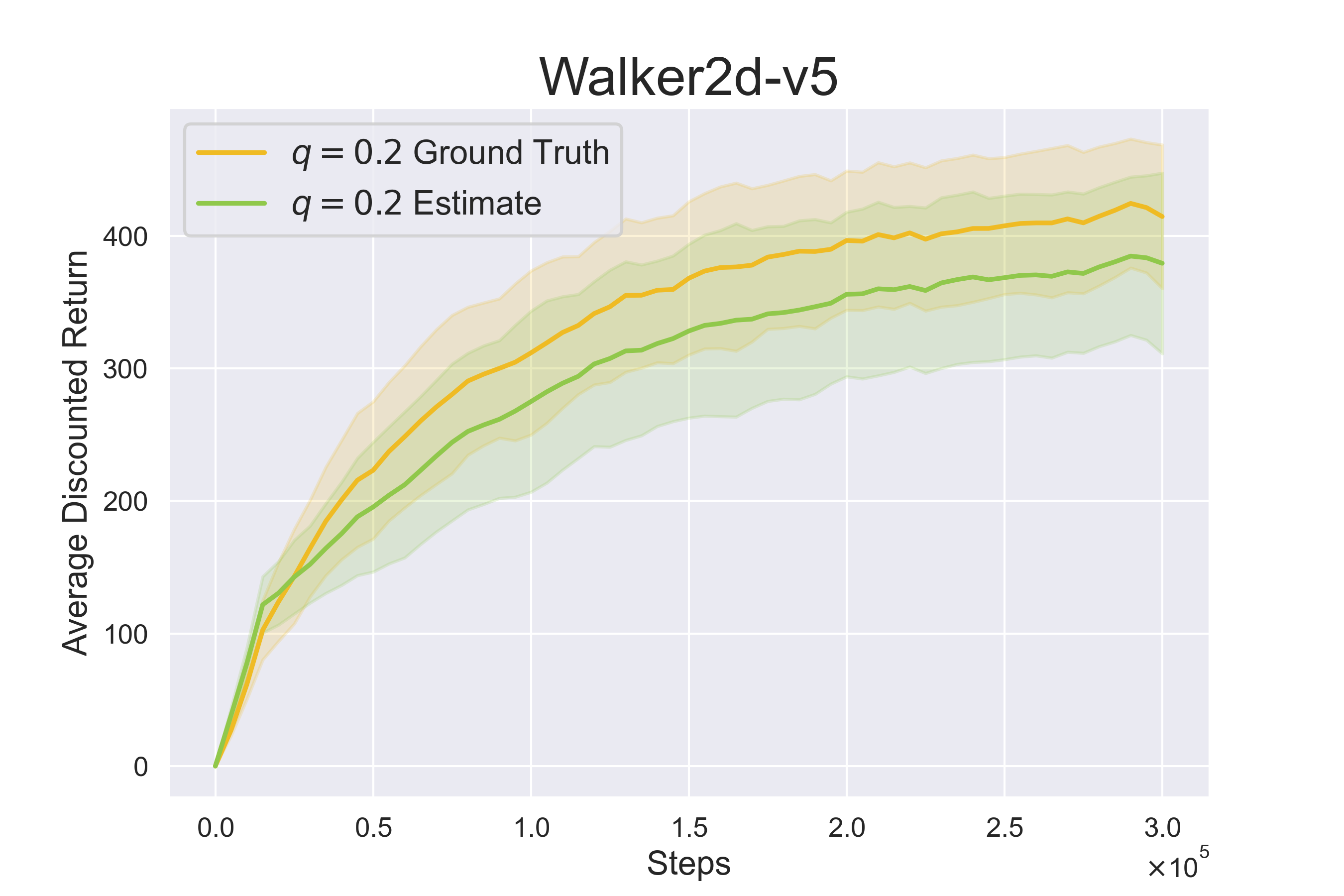}\includegraphics[width=0.2\columnwidth]{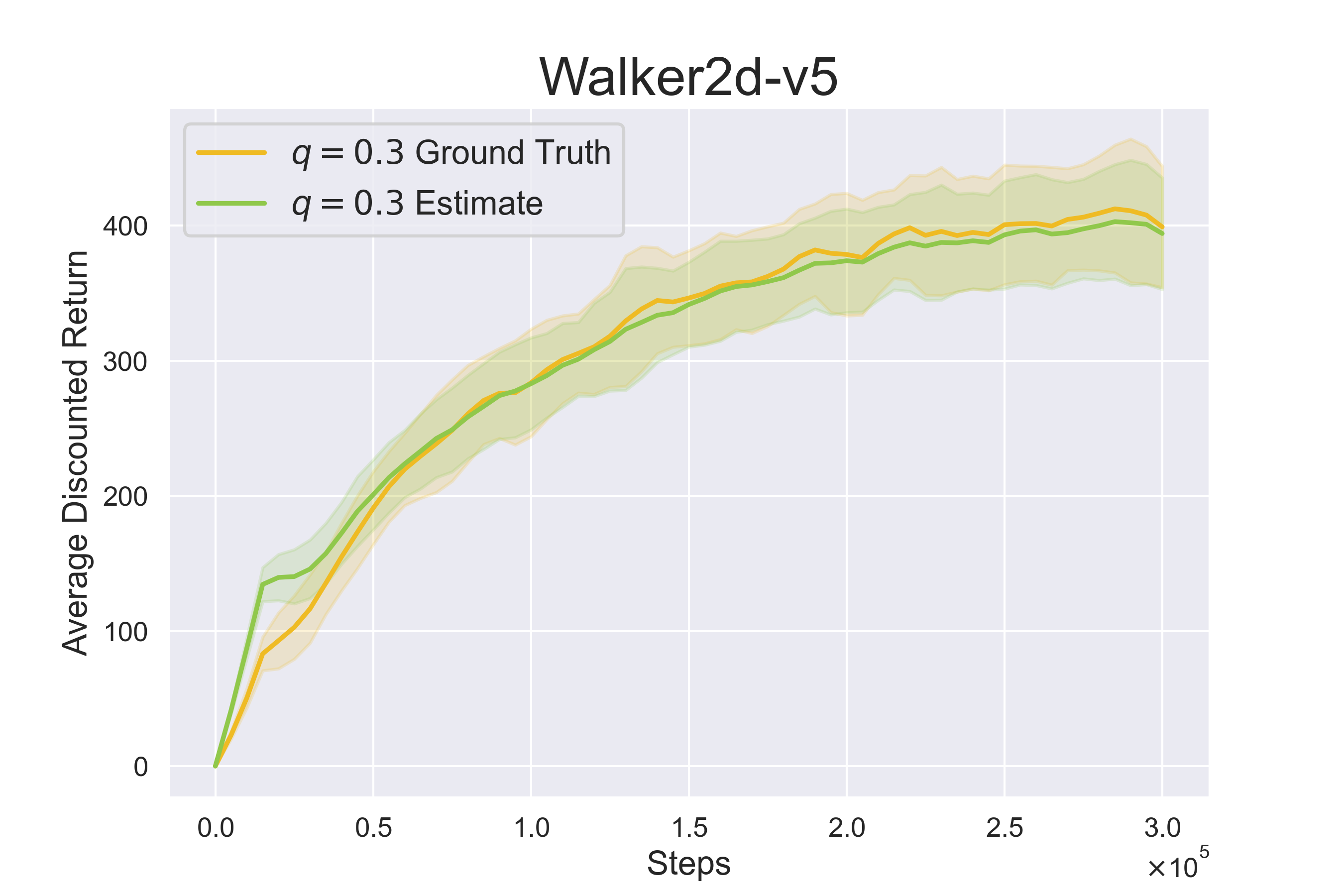}\includegraphics[width=0.2\columnwidth]{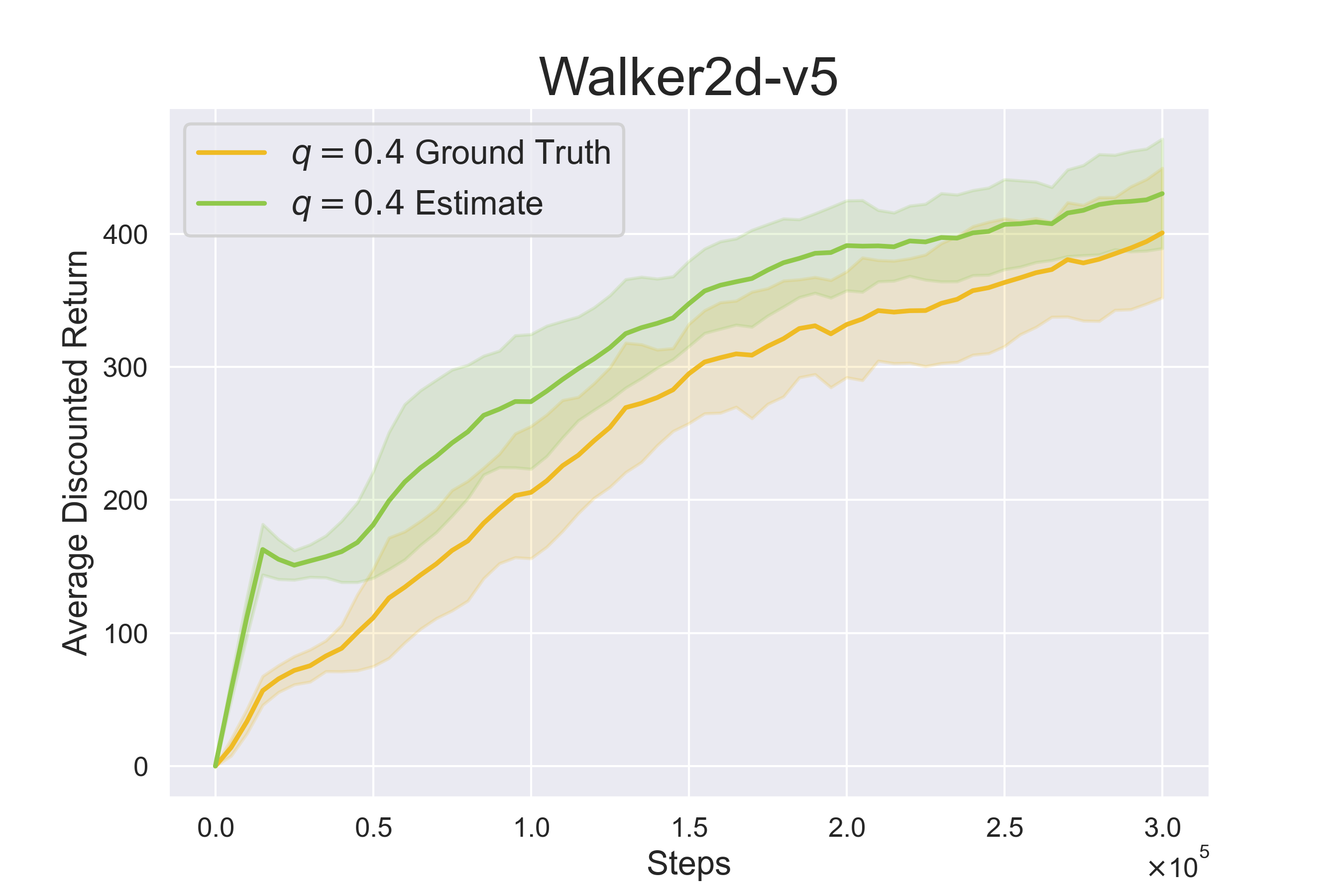}\includegraphics[width=0.2\columnwidth]{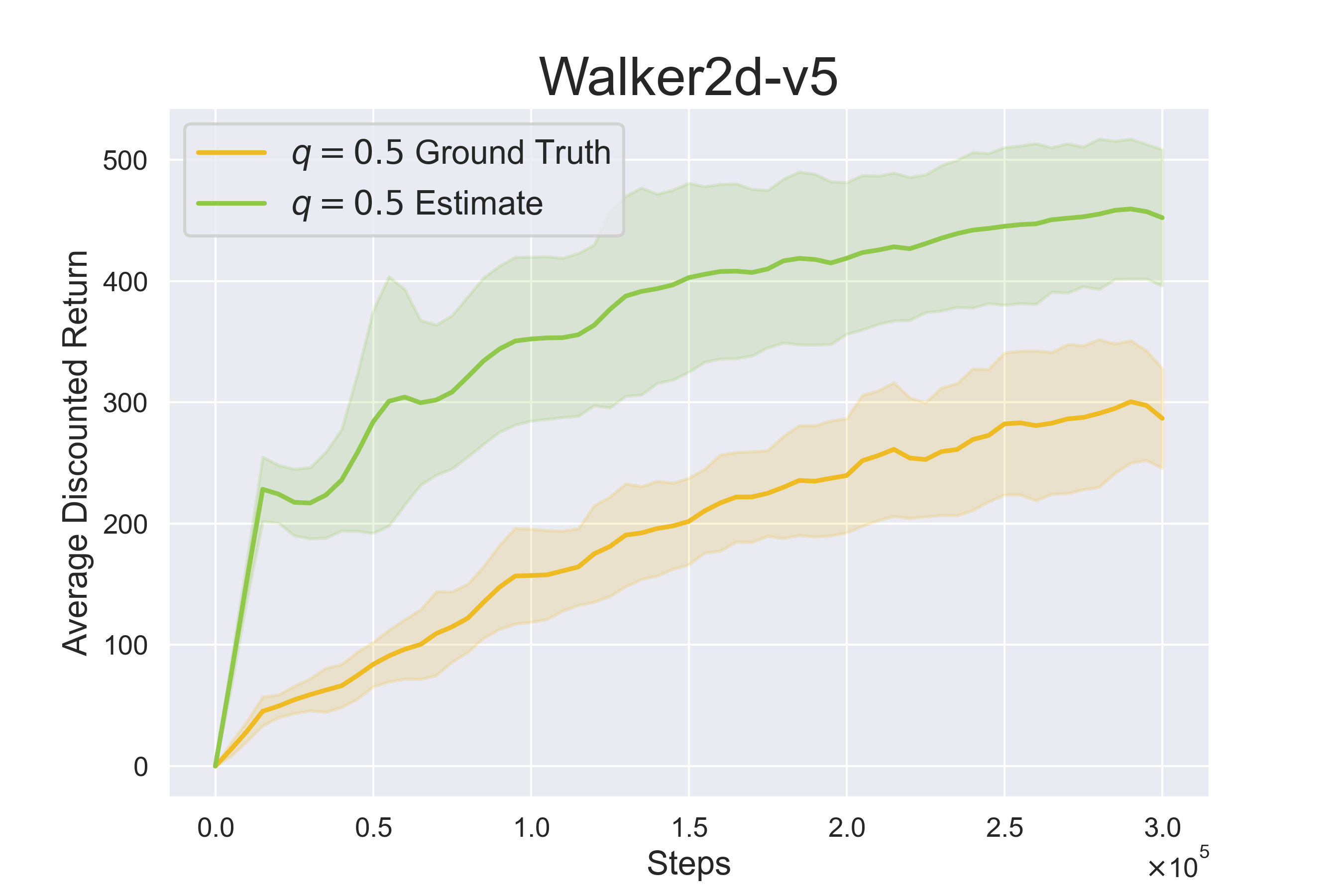}
\par\end{centering}
\centering{}\caption{\label{fig:Error-quantile}Estimated and ground-truth average discounted
return against environment steps for the MAMC with different quantile
parameters $q$ on HalfCheetah-v5 and Walker2d-v5}
\end{figure}

Figure \ref{fig:Error-quantile} plots the estimated and ground-truth
average discounted return against environment steps for MAMC with
different quantile parameters $q$ on HalfCheetah-v5 and Walker2d-v5.
It is obvious that the estimated value increases as $q$ increases;
the best $q$ in terms of the smallest distance to ground-truth value
is 0.4 for HalfCheetah-v5 and 0.3 for Walker2d-v5. However, the values
are inconsistent to the best $q$ in terms of the average return,
which is 0.5 for HalfCheetah-v5 and 0.2 for Walker2d-v5 (cf.~Table
\ref{tab:mean-std-q}). That is, the setting of quantile parameters
$q$ should also considers the environmental preferences of optimism
and pessimism. In general, a range between 0.2 and 0.3 is a good setting
for environments which favor pessimism, and a value between 0.3 and
0.4 is nice for optimism cases; thus, a robust $q$ value may near
0.3, but the best one for a specific environment still needs to be
investigated. 

\subsection{Validity Analysis\label{subsec:Validity-analysis}}

The proposed MAMC is based on deterministic policy, and the results
have shown that the MAMC can ameliorate the performance of TD3-SMR
and REDQ-SMR. The MAMC is also beneficial in comparison to SAC-SMR,
a simple but powerful method with stochastic policy. Past studies
have discovered the potential of stochastic policy over deterministic
policy, and this may be the weakness of the MAMC, which is considered
as the main reason to be surpassed by REDQ-SMR. 

\section{Conclusions\label{sec:Conclusions}}

This study proposes a multi-actor multi-critic deep deterministic
reinforcement learning method. The MAMC includes a selection of actors
for exploration using skill and creativity factors, an ensemble target
value based on a predefined quantile parameter, and a selection of
best actor regarding skill factor for exploitation. Theoretical analysis
proves the MAMC having bounded estimation error, and learning stability
over SAMC and MASC. From experimental results, the MAMC excels TD3-SMR,
DARC-SMR, and SAC-SMR with better quality and faster convergence on
the selected environments in MuJoCo. The validity analysis shows a
weakness of deterministic based method and is also a possible future
extension. Another promising orientation for future research is to
adapt the quantile parameter to address the issue of estimation accuracy
by balancing optimism and pessimism.%

{
\bibliographystyle{plain}
\medskip{}
\bibliography{MAMC2}
}

\appendix

\section{Experimental Settings in Detail\label{sec:Detailed-Experimental-Settings}}

This section gives detailed experimental settings adopted in this
study. The code along with the instructions containing the exact command
and environment needed to run to reproduce the results, and the followed
licenses are available at \url{https://github.com/AndyWu101/MAMC}.

\begin{table}[!t]
\caption{\label{tab:Hyperparameter-settings}Hyperparameter settings}

\centering{}%
\begin{tabular}{llccccc}
\toprule 
{\small Type} & {\small Hyperparameter} & {\small TD3-SMR} & {\small DARC-SMR} & {\small SAC-SMR} & {\small REDQ-SMR} & {\small MAMC}\tabularnewline
\midrule
{\small Shared} & {\small\#Actors ($N_{A}$)} & {\small 1} & {\small 2} & {\small 1} & {\small 1} & {\small 10}\tabularnewline
 & {\small\#Critics ($N_{C}$)} & {\small 2} & 2 & {\small 2} & {\small 10} & {\small 10}\tabularnewline
 & {\small Discount factor} & {\small 0.99} & {\small 0.99} & {\small 0.99} & {\small 0.99} & {\small 0.99}\tabularnewline
 & {\small Actor learning rate} & {\small 3.0E-4} & {\small 3.0E-4} & {\small 3.0E-4} & {\small 3.0E-4} & {\small 1.0E-4}{\small\tablefootnote{{\small Without delayed update and target actor, the MAMC adopts a
small learning rate.}}}\tabularnewline
 & {\small Critic learning rate} & {\small 3.0E-4} & {\small 3.0E-4} & {\small 3.0E-4} & {\small 3.0E-4} & {\small 3.0E-4}\tabularnewline
 & {\small Optimizer} & {\small Adam} & {\small Adam} & {\small Adam} & {\small Adam} & {\small Adam}\tabularnewline
 & {\small Batch size ($N_{\mathcal{B}}$)} & {\small 256} & {\small 256} & {\small 256} & {\small 256} & {\small 256}\tabularnewline
 & {\small Actor target} & {\small v} & {\small v} & {\small -} & {\small -} & {\small -}\tabularnewline
 & {\small Critic target} & {\small v} & {\small v} & {\small v} & {\small v} & {\small v}\tabularnewline
 & {\small Soft update ratio ($\tau$)} & {\small 5.0E-3} & {\small 5.0E-3} & {\small 5.0E-3} & {\small 5.0E-3} & {\small 5.0E-3}\tabularnewline
 & {\small SMR ratio ($M$)} & {\small 10} & {\small 10} & {\small 10} & {\small 10} & {\small 10}\tabularnewline
 & {\small Warm-up steps} & {\small 5k} & {\small 5k} & {\small 5k} & {\small 5k} & {\small 5k}\tabularnewline
 & {\small Delayed update ($d$)} & {\small 2} & {\small 1} & {\small 1} & {\small 10} & {\small 1}\tabularnewline
\midrule 
{\small Deterministic} & {\small Exploration noise} & {\small$\mathcal{N}(0,0.1)$} & {\small$\mathcal{N}(0,0.1)$} & {\small -} & {\small -} & {\small$\mathcal{N}(0,0.1)$}\tabularnewline
 & {\small Target policy noise} & {\small$\mathcal{N}(0,0.2)$} & {\small$\mathcal{N}(0,0.2)$} & {\small -} & {\small -} & {\small$\mathcal{N}(0,0.1)$}\tabularnewline
 & {\small Noise clip} & {\small$[-0.5,0.5]$} & {\small$[-0.5,0.5]$} & {\small -} & {\small -} & {\small -}\tabularnewline
\midrule 
{\small Stochastic} & {\small Temperature ($\alpha$)} & {\small -} & {\small -} & {\small Tuned}{\small\tablefootnote{{\small SAC set the $\alpha$ to 0.05 for Humanoid, and 0.2 for the
others.}}} & {\small Adaptive} & {\small -}\tabularnewline
 & {\small Log std. clip} & {\small -} & {\small -} & {\small$[-20,2]$} & {\small$[-20,2]$} & {\small -}\tabularnewline
\midrule 
{\small Specific} & {\small Weighting coef. ($\nu$)} & {\small -} & {\small Tuned}{\small\tablefootnote{{\small DARC set the $\nu$ to 0.15 for Hopper, 0.25 for Ant, and 0.1
for the others.}}} & {\small -} & {\small -} & {\small -}\tabularnewline
 & {\small Regularization ($\lambda$)} & {\small -} & {\small 5.0E-3} & {\small -} & {\small -} & {\small -}\tabularnewline
 & {\small Target entropy} & {\small -} & {\small -} & {\small -} & {\small Tuned}{\small\tablefootnote{{\small REDQ set target entropy to -1 for Hopper, -2 for Humanoid,
-3 for HalfCheetah and Walker, and -4 for Ant.}}} & {\small -}\tabularnewline
 & {\small Ensemble subset size} & {\small -} & {\small -} & {\small -} & {\small 2} & {\small -}\tabularnewline
 & {\small Quantile ($q$)} & {\small -} & {\small -} & {\small -} & {\small -} & {\small 0.2}\tabularnewline
\bottomrule
\end{tabular}
\end{table}

\subsection{Hyperparameter Settings\label{subsec:Hyperparameters}}

Table \ref{tab:Hyperparameter-settings} compiles the hyperparameter
settings for the three deterministic-policy-based (TD3-SMR, DARC-SMR,
and MAMC) and two stochastic-policy-based (SAC-SMR and REDQ-SMR) methods.
Most of the settings follow the original suggestions in the non-SMR
version. In the shared hyperparameters, the number of actors and critics
in the MAMC are both set to 10, which equals to the number of critics
in REDQ-SMR. In addition, the DARC-SMR, SAC-SMR, and MAMC have no
delayed update for each actor, whilst TD3-SMR and REDQ-SMR has a delayed
update of 2 and 10, respectively. Furthermore, SAC-SMR, REDQ-SMR,
and the MAMC do not consider the utilization of actor target when
calculating the TD target. Noteworthily, this study sets a low actor
learning rate for the MAMC since it has no delayed update and actor
target. All the test methods have an SMR ratio of 10. As REDQ-SMR
has considered SMR technique, its UTD ratio is set to 1 for a fair
comparison.

For hyperparameters considered in deterministic-policy-based methods,
the proposed MAMC adds noise to actors when exploration and calculation
of target values with the same distribution, while TD3-SMR and DARC-SMR
considered larger noise when computing the target values than exploration.
Also, the MAMC has no noise clip for simplicity. As for hyperparameters
leveraged in stochastic-policy-based methods, the SAC-SMR set a small
temperature for Humanoid, and a large one for the others, and the
REDQ-SMR considered an adaptive control of temperature.

Some hyperparameters are exploited in a specific method. DARC-SMR
fine-tuned weighting coefficient $\nu$ for different environment,
and considered a regularization coefficient for similarity of two
critics. REDQ-SMR also fine-tuned the target entropy for each environment,
and set the ensemble subset size to 2. For the MAMC, the number of
actors and critics are both set to 10, and the quantile parameter
$q$ is set to 0.2.

\subsection{System Configuration\label{subsec:System-Configuration}}

All the experiments are conducted on a server with Intel Xeon W7-2475X
CPU (with 2.6 GHz clock rate, 20 cores and 40 hyperthreads), two NVIDIA
RTX 4090 GPU cards (each with 24GB memory), and 128 GB main memory.

\subsection{MuJoCo\label{subsec:MuJoco}}

The properties of the selected environments in MuJoCo \cite{MuJoCo}
are listed as follows:
\begin{itemize}
\item Hopper-v5
\begin{itemize}
\item Appearance: 2D single-leg hopping robot
\begin{itemize}
\item Simulation: kangaroo hopping
\item State: 11-dimensional random vector $s\in\mathbb{R}^{11}$, includes
position and velocity information of various body parts
\item Action: 3-dimensional random vector $a\in[-1,1]^{3}$, corresponding
to torque control of three hinge joints
\end{itemize}
\item HalfCheetah-v5
\begin{itemize}
\item Appearance: 2D bipedal robot
\item Simulation: cheetah running
\item State: 17-dimensional random vector $s\in\mathbb{R}^{17}$, includes
joint angles, angular velocities, and body linear velocity
\item Action: 6-dimensional random vector $a\in[-1,1]^{6}$, corresponding
to torque control of six hinge joints
\end{itemize}
\item Walker2d-v5
\begin{itemize}
\item Appearance: 2D bipedal walking robot
\item Simulation: human walking
\item State: 17-dimensional random vector $s\in\mathbb{R}^{17}$, includes
position and velocity information of various body parts
\item Action: 6-dimensional random vector $a\in[-1,1]^{6}$, corresponding
to torque control of six hinge joints
\end{itemize}
\item Ant-v5
\begin{itemize}
\item Appearance: 3D quadrupedal robot
\item Simulation: ant walking
\item State: 105-dimensional random vector $s\in\mathbb{R}^{105}$, includes
position, velocity, and angle information of various body parts
\item Action: 8-dimensional random vector $a\in[-1,1]^{8}$, corresponding
to torque control of eight hinge joints
\end{itemize}
\item Humanoid-v5
\begin{itemize}
\item Appearance: 3D bipedal humanoid robot
\item Simulation: complex human-like locomotion and balancing
\item State: 348-dimensional random vector $s\in\mathbb{R}^{348}$, includes
joint angles, velocities, torso orientation, and center of mass information
\item Action: 17-dimensional random vector $a\in[-0.4,0.4]^{17}$, corresponding
to torque control of 17 motor joints
\end{itemize}
\end{itemize}

\section{Proof of Theorems\label{sec:Proof}}

\setcounter{theorem}{0}

\begin{theorem} The variance of target values obtained by multiple
actors are less than that using a single actor
\begin{equation}
\mathbb{V}[\hat{V}_{A}(s^{\prime};C^{\prime})]\leq\mathbb{V}[\hat{V}_{\phi}(s^{\prime};C^{\prime})]\,.
\end{equation}

\end{theorem}

\begin{proof} Assume that the distribution of $\{\hat{V}_{\phi_{i}}(s^{\prime};C^{\prime})\}_{1\leq i\leq N_{A}}$
are not skewed (symmetric), we have:

\begin{align}
\mathbb{V}_{s^{\prime}\sim S}[\hat{V}_{A}(s^{\prime};C^{\prime})] & =\mathbb{V}[\mathrm{Med}(\{\hat{V}_{\phi_{i}}(s^{\prime};C^{\prime})\}_{1\leq i\leq N_{A}})]\nonumber \\
 & =\mathbb{V}[\mathbb{E}_{\phi_{i}\in A}[\hat{V}_{\phi_{i}}(s^{\prime};C^{\prime})]]\nonumber \\
 & =\mathbb{V}[N_{A}^{-1}{\textstyle \sum_{\phi_{i}\in A}}\hat{V}_{\phi_{i}}(s^{\prime};C^{\prime})]\nonumber \\
 & =N_{A}^{-2}{\textstyle \sum_{\phi_{i}\in A}}\mathbb{V}[\hat{V}_{\phi_{i}}(s^{\prime};C^{\prime})]\nonumber \\
 & \leq N_{A}^{-1}\mathbb{V}[\hat{V}_{\phi_{\max}}(s^{\prime};C^{\prime})]\nonumber \\
 & \leq\mathbb{V}[\hat{V}_{\phi_{\min}}(s^{\prime};C^{\prime})]\nonumber \\
 & \leq\mathbb{V}[\hat{V}_{\phi}(s^{\prime};C^{\prime})]\nonumber \\
 & \leq\mathbb{V}[\hat{V}_{\phi_{\max}}(s^{\prime};C^{\prime})]\,.\label{eq:proof-v-A}
\end{align}

The inequality is always satisfied comparing to $\phi=\phi_{\max}$.
For generalization to any arbitrary $\phi\geq\phi_{\min}$, the ratio
of maximum to minimum variance are within some bound
\begin{equation}
\nicefrac{\mathbb{V}_{\phi_{\max}}}{\mathbb{V}_{\phi_{\min}}}\leq\epsilon_{A}\,,
\end{equation}
 where $\epsilon_{A}=N_{A}$ serves as a constraint. Also, it is apparent
that the larger the $N_{A}$ the easier the satisfaction of the constraint
on the ratio.

\end{proof}

\begin{theorem} The variance of target values obtained by multiple
critics are less than using a single critic
\begin{equation}
\mathbb{V}[\hat{V}_{\phi}(s^{\prime};C^{\prime})]\leq\mathbb{V}[\hat{V}_{\phi}(s^{\prime};\theta^{\prime})]\,.
\end{equation}

\end{theorem}

\begin{proof} Assume that the $q$-th quantile among critic targets
$C^{\prime}$ is $c_{q}$ times their expectation:
\begin{gather}
\begin{aligned}\hat{V}_{\phi}(s^{\prime};C^{\prime}) & =\mathrm{Quantile}_{q}(\{Q_{\theta_{j}^{\prime}}(s^{\prime},\pi_{\phi}(s^{\prime}))\}_{1\leq j\leq N_{C}})\\
 & =c_{q}\mathbb{E}_{\theta^{\prime}\in C^{\prime}}[Q_{\theta^{\prime}}(s^{\prime},\pi_{\phi}(s^{\prime}))]\;\exists\;c_{q}\in\mathbb{R}
\end{aligned}
\,,\label{eq:assumption-quantile}
\end{gather}
 and thus the following equation proves the theorem:

\begin{align}
\mathbb{V}_{s^{\prime}\sim S}[\hat{V}_{\phi}(s^{\prime};C^{\prime})] & =\mathbb{V}[c_{q}\mathbb{E}_{\theta^{\prime}\in C^{\prime}}[Q_{\theta^{\prime}}(s^{\prime},\pi_{\phi}(s^{\prime}))]]\nonumber \\
 & =c_{q}^{2}\mathbb{V}[N_{C}^{-1}{\textstyle \sum_{\theta^{\prime}\in C^{\prime}}}Q_{\theta^{\prime}}(s^{\prime},\pi_{\phi}(s^{\prime}))]\nonumber \\
 & =c_{q}^{2}N_{C}^{-2}{\textstyle \sum_{\theta^{\prime}\in C^{\prime}}}\mathbb{V}[Q_{\theta^{\prime}}(s^{\prime},\pi_{\phi}(s^{\prime}))]\nonumber \\
 & \leq c_{q}^{2}N_{C}^{-1}\mathbb{V}[Q_{\theta_{\max}^{\prime}}(s^{\prime},\pi_{\phi}(s^{\prime}))]\nonumber \\
 & \leq\mathbb{V}[Q_{\theta_{\min}^{\prime}}(s^{\prime},\pi_{\phi}(s^{\prime}))]\nonumber \\
 & =\mathbb{V}[\hat{V}_{\phi}(s^{\prime};\theta_{\min}^{\prime})]\nonumber \\
 & \leq\mathbb{V}[\hat{V}_{\phi}(s^{\prime};\theta^{\prime})]\nonumber \\
 & \leq\mathbb{V}[\hat{V}_{\phi}(s^{\prime};\theta_{\max}^{\prime})]\,.\label{eq:proof-v-c}
\end{align}
This theorem holds when the ratio of maximum to minimum variance are
within some bound
\begin{equation}
\nicefrac{\mathbb{V}_{\theta_{\max}^{\prime}}}{\mathbb{V}_{\theta_{\min}^{\prime}}}\leq\epsilon_{C}\,,
\end{equation}
 subject to
\begin{equation}
\epsilon_{C}=c_{q}^{-2}N_{C}\,.
\end{equation}
 The bound $\epsilon_{C}$ can be viewed as a constraint of SAMC to
be more stable than SASC. From the above equation, it is obvious that
the intensity of the constraint is proportional to the coefficient
$c_{q}$ and is inverse proportional to the number of critics.

\end{proof}

For proving the next theorems, this study first introduces two lemmas.

\begin{lemma} The target values among multiple actors are in between
the minimum and maximum of target values for a single actor

\begin{equation}
\mathbb{E}[\hat{V}_{\phi_{\min}}(s^{\prime};C)]\leq\mathbb{E}[\hat{V}_{A}(s^{\prime};C)]\leq\mathbb{E}[\hat{V}_{\phi_{\max}}(s^{\prime};C)]\,.
\end{equation}

\end{lemma}

\begin{proof} The lemma holds owing to the following inequality:

\begin{equation}
\hat{V}_{\phi_{\min}}(s^{\prime};C)\leq\hat{V}_{A}(s^{\prime};C)\leq\hat{V}_{\phi_{\max}}(s^{\prime};C)\,.
\end{equation}

\end{proof}

\begin{lemma} The target values among multiple critics are in between
the minimum and maximum of target values for a single critic

\begin{equation}
\mathbb{E}[\hat{V}_{A}(s^{\prime};\theta_{\min})]\leq\mathbb{E}[\hat{V}_{A}(s^{\prime};C)]\leq\mathbb{E}[\hat{V}_{A}(s^{\prime};\theta_{\max})]\,.
\end{equation}

\end{lemma}

\begin{proof} Similarly, the inequality holds with

\begin{equation}
\hat{V}_{A}(s^{\prime};\theta_{\min})\leq\hat{V}_{A}(s^{\prime};C)\leq\hat{V}_{A}(s^{\prime};\theta_{\max})\,.
\end{equation}

\end{proof}

\begin{theorem} The estimation error of MAMC is between the estimation
error of multiple actors with minimum and maximum critics
\begin{equation}
\mathcal{E}_{A,Q_{\theta_{\min}}}\leq\mathcal{E}_{A,C}\leq\mathcal{E}_{A,Q_{\theta_{\max}}}\,.
\end{equation}

\end{theorem}

\begin{proof} The proof is similar to the one given in \cite{DARC_AAAI2022}:

\begin{align}
\mathcal{E}_{A,Q_{\theta_{\min}}} & =\mathbb{E}[\hat{V}_{A}(s^{\prime};\theta_{\min})]-\mathbb{E}[V_{\phi^{*}}(s^{\prime})]\nonumber \\
 & \leq\mathbb{E}[\hat{V}_{A}(s^{\prime};C)]-\mathbb{E}[V_{\phi^{*}}(s^{\prime})]\nonumber \\
 & =\mathcal{E}_{A,C}\nonumber \\
 & \leq\mathbb{E}[\hat{V}_{A}(s^{\prime};\theta_{\max})]-\mathbb{E}[V_{\phi^{*}}(s^{\prime})]\nonumber \\
 & =\mathcal{E}_{A,Q_{\theta_{\max}}}\,.\label{eq:proof-e-a}
\end{align}

\end{proof}

\begin{theorem} The estimation error of MAMC is between the estimation
error of multiple critics with minimum and maximum actors
\begin{equation}
\mathcal{E}_{\pi_{\phi_{\min}},C}\leq\mathcal{E}_{A,C}\leq\mathcal{E}_{\pi_{\phi_{\max}},C}\,.
\end{equation}

\end{theorem}

\begin{proof} Similar derivation can be applied:

\begin{align}
\mathcal{E}_{\pi_{\phi_{\min}},C} & =\mathbb{E}[\hat{V}_{\phi_{\min}}(s^{\prime};C)]-\mathbb{E}[V_{\phi^{*}}(s^{\prime})]\nonumber \\
 & \leq\mathbb{E}[\hat{V}_{A}(s^{\prime};C)]-\mathbb{E}[V_{\phi^{*}}(s^{\prime})]\nonumber \\
 & =\mathcal{E}_{A,C}\nonumber \\
 & \leq\mathbb{E}[\hat{V}_{\phi_{\max}}(s^{\prime};C)]-\mathbb{E}[V_{\phi^{*}}(s^{\prime})]\nonumber \\
 & =\mathcal{E}_{\pi_{\phi_{\max}},C}\,,\label{eq:proof-e-c}
\end{align}
 and the theorem is proved.

\end{proof}

\begin{table}[!t]
\caption{\label{tab:Wilcoxon-signed-rank-quality-S}Wilcoxon signed rank test
for TD3 and DARC compared with the MAMC at early (100k), middle (200k),
and late stage (300k). The win/tie/lose denotes the number of environments
that the MAMC is significantly superior (+), equal (\textasciitilde ),
and inferior (-) to a corresponding test method.}

\centering{}%
\begin{tabular}{llrrrr}
\toprule 
Stage & $p$-value & \multicolumn{1}{c}{TD3-SMR} & \multicolumn{1}{c}{DARC-SMR} & \multicolumn{1}{c}{SAC-SMR} & \multicolumn{1}{c}{REDQ-SMR}\tabularnewline
\midrule 
100k & Hopper-v5 & 5.27E-02 (\textasciitilde ) & 2.44E-02 ($-$) & 1.61E-01 (\textasciitilde ) & 2.44E-02 (-)\tabularnewline
 & HalfCheetah-v5 & 1.38E-01 (\textasciitilde ) & 6.88E-01 (\textasciitilde ) & 6.15E-01 (\textasciitilde ) & 4.61E-01 (\textasciitilde )\tabularnewline
 & Walker2d-v5 & 4.88E-03 (+) & 3.22E-02 (+) & 1.37E-02 (+) & 3.85E-01 (\textasciitilde )\tabularnewline
 & Ant-v5 & 9.77E-04 (+) & 9.77E-04 (+) & 2.93E-03 (+) & 4.23E-01 (\textasciitilde )\tabularnewline
 & Humanoid-v5 & 9.77E-03 (+) & 9.67E-02 (\textasciitilde ) & 2.44E-02 (+) & 5.00E-01 (\textasciitilde )\tabularnewline
\midrule 
\multicolumn{2}{l}{Summary (win/tie/lose)} & 3/2/0 & 2/2/1 & 3/2/0 & 0/4/1\tabularnewline
\midrule
200k & Hopper-v5 & 6.54E-02 (\textasciitilde ) & 9.67E-02 (\textasciitilde ) & 1.86E-02 (-) & 9.77E-04 (-)\tabularnewline
 & HalfCheetah-v5 & 5.77E-01 (\textasciitilde ) & 1.88E-01 (\textasciitilde ) & 5.27E-02 (\textasciitilde ) & 5.27E-02 (\textasciitilde )\tabularnewline
 & Walker2d-v5 & 3.48E-01 (\textasciitilde ) & 2.78E-01 (\textasciitilde ) & 9.67E-02 (\textasciitilde ) & 2.46E-01 (\textasciitilde )\tabularnewline
 & Ant-v5 & 9.77E-04 (+) & 9.77E-04 (+) & 1.95E-03 (+) & 9.67E-02 (\textasciitilde )\tabularnewline
 & Humanoid-v5 & 4.88E-03 (+) & 1.88E-01 (\textasciitilde ) & 4.20E-02 (+) & 5.00E-01 (\textasciitilde )\tabularnewline
\midrule 
\multicolumn{2}{l}{Summary (win/tie/lose)} & 2/3/0 & 1/4/0 & 2/2/1 & 0/4/1\tabularnewline
\midrule
300k & Hopper-v5 & 1.38E-01 (\textasciitilde ) & 5.39E-01 (\textasciitilde ) & 2.78E-01 (\textasciitilde ) & 8.01E-02 (\textasciitilde )\tabularnewline
 & HalfCheetah-v5 & 6.88E-01 (\textasciitilde ) & 1.38E-01 (\textasciitilde ) & 5.27E-02 (\textasciitilde ) & 9.77E-03 (-)\tabularnewline
 & Walker2d-v5 & 3.13E-01 (\textasciitilde ) & 4.20E-02 (-) & 1.61E-01 (\textasciitilde ) & 2.46E-01 (\textasciitilde )\tabularnewline
 & Ant-v5 & 9.77E-04 (+) & 1.95E-03 (+) & 1.95E-03 (+) & 8.01E-02 (\textasciitilde )\tabularnewline
 & Humanoid-v5 & 1.37E-02 (+) & 5.77E-01 (\textasciitilde ) & 4.61E-01 (\textasciitilde ) & 4.20E-02 (-)\tabularnewline
\midrule 
\multicolumn{2}{l}{Summary (win/tie/lose)} & 2/3/0 & 1/3/1 & 1/4/0 & 0/3/2\tabularnewline
\bottomrule
\end{tabular}
\end{table}

\section{Additional Experimental Results}

Additional experimental results and further analysis are given in
the following subsections. 

\subsection{Statistical Analysis\label{subsec:Mean-and-stdandard}}

Table \ref{tab:Wilcoxon-signed-rank-quality-S} compiles the Wilcoxon
signed rank test for TD3 and DARC compared with the MAMC at early
(100k), middle (200k), and late stage (300k). The win/tie/lose denotes
the number of environments that the MAMC is significantly superior
(+), equal (\textasciitilde ), and inferior (-) to a corresponding
test method. The MAMC betters TD3-SMR, DARC-SMR, and SAC-SMR at all
three stage. In addition, the MAMC is comparable to REDQ-SMR at early
and middle stages, yet is inferior to the REDQ-SMR at late stage.

\begin{figure}[!b]
\includegraphics[width=0.48\columnwidth]{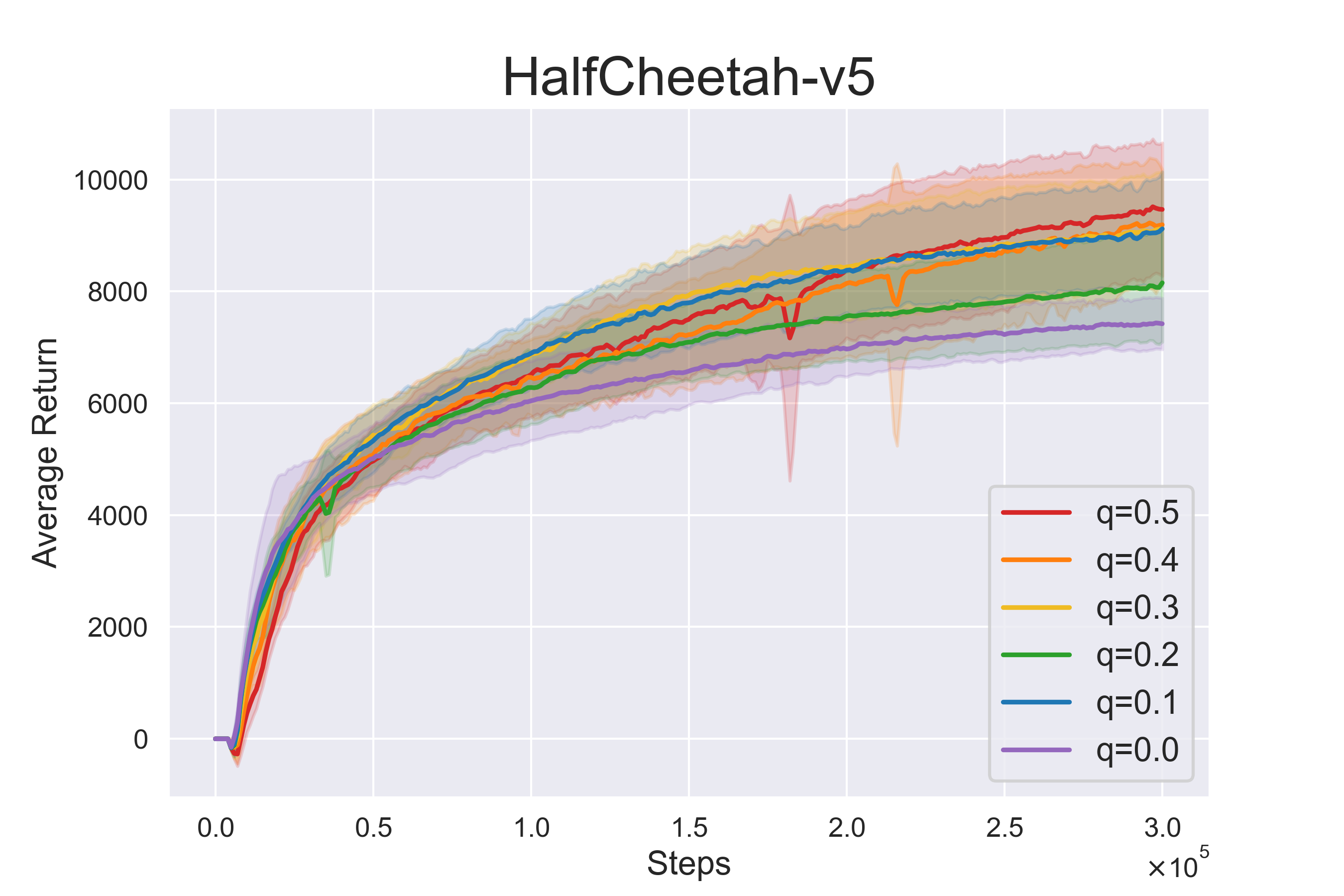}~~~\includegraphics[width=0.48\columnwidth]{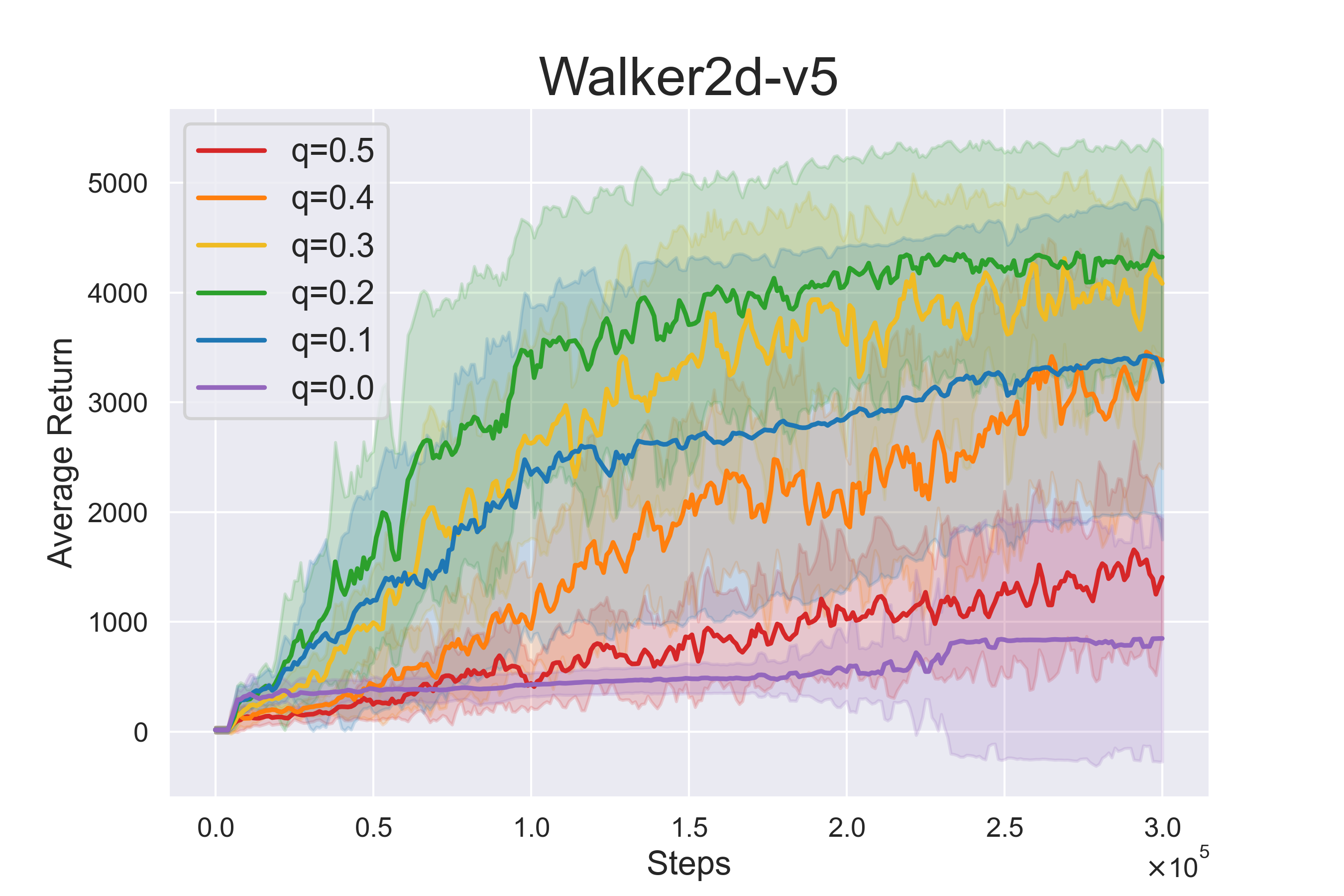}
\centering{}\caption{\label{fig:Average-return-quantile}Average return against environmental
steps for MAMC with different quantile parameters $q\in\{0.0,0.1,0.2,0.3,0.4,0.5\}$
on HalfCheetah-v5 and Walker2d-v5}
\end{figure}

\subsection{Quantile Value Comparison\label{subsec:Quantile}}

Figure \ref{fig:Average-return-quantile} draws the average return
against environment steps for MAMC with different quantile parameters
$q\in\{0.0,0.1,0.2,0.3,0.4,0.5\}$ on HalfCheetah-v5 and Walker2d-v5.
On HalfCheetah-v5, the MAMCs with $q=0.1$, 0.3, and 0.5 are better,
while on Walker2d-v5 the MAMCs with $q=0.2$, and 0.3 performs nicer.
The quantile parameter highly hinges on the environmental preference
of optimism or pessimism. From the experimental results, this study
would suggest setting $q\in[0.2,0.3]$ for better robustness.

\subsection{The number of Actors and Critics\label{subsec:The-number-of-AC}}

Figure \ref{fig:Average-return-actors-critics} depicts the average
return against environmental steps for MAMC with different number
of actors $N_{A}\in\{2,5,10,15\}$ and critics $N_{C}\in\{2,5,10,15\}$
on HalfCheetah-v5, Ant-v5, and Humanoid-v5 over five trials. For setting
the number of actors, the MAMC with $N_{A}=10$ performs best, and
the performance deteriorates as the number of actors grows to 15 or
shrinks to 5 and 2. By varying the number of critics, the MAMC with
$N_{C}=10$ provides the most robust results on the three environments,
in comparison to the other three values. Hence, this study suggests
taking 10 actors and critics for the MAMC.

\begin{figure}[!t]
\begin{centering}
\includegraphics[width=0.96\columnwidth]{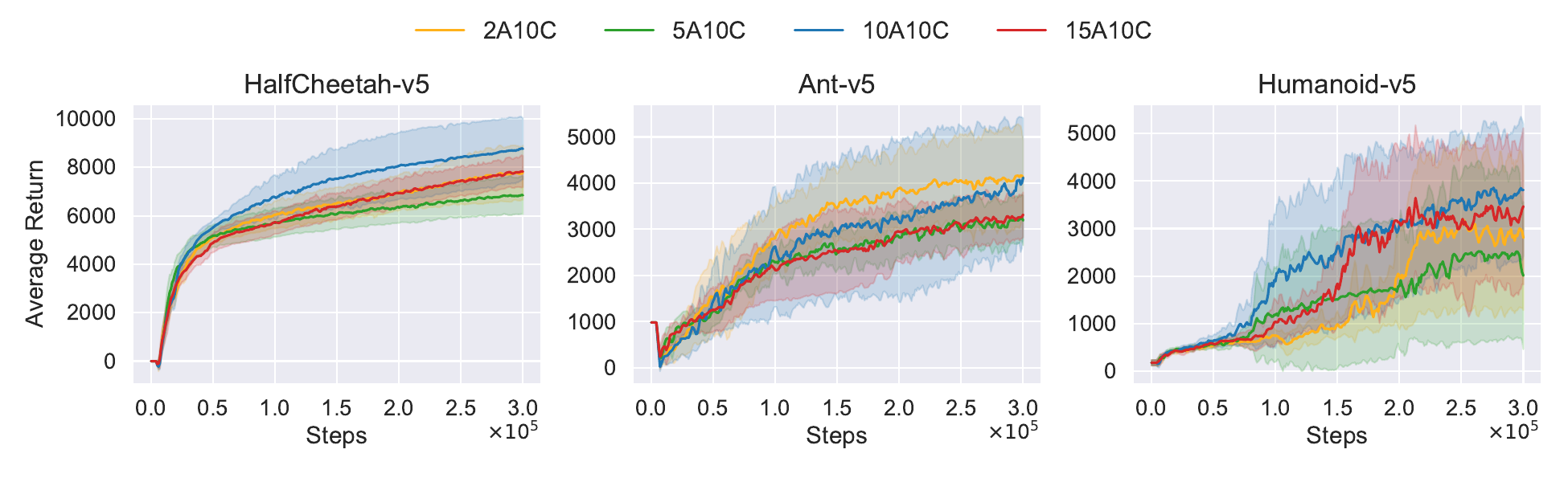}
\par\end{centering}
\centering{}\includegraphics[width=0.96\columnwidth]{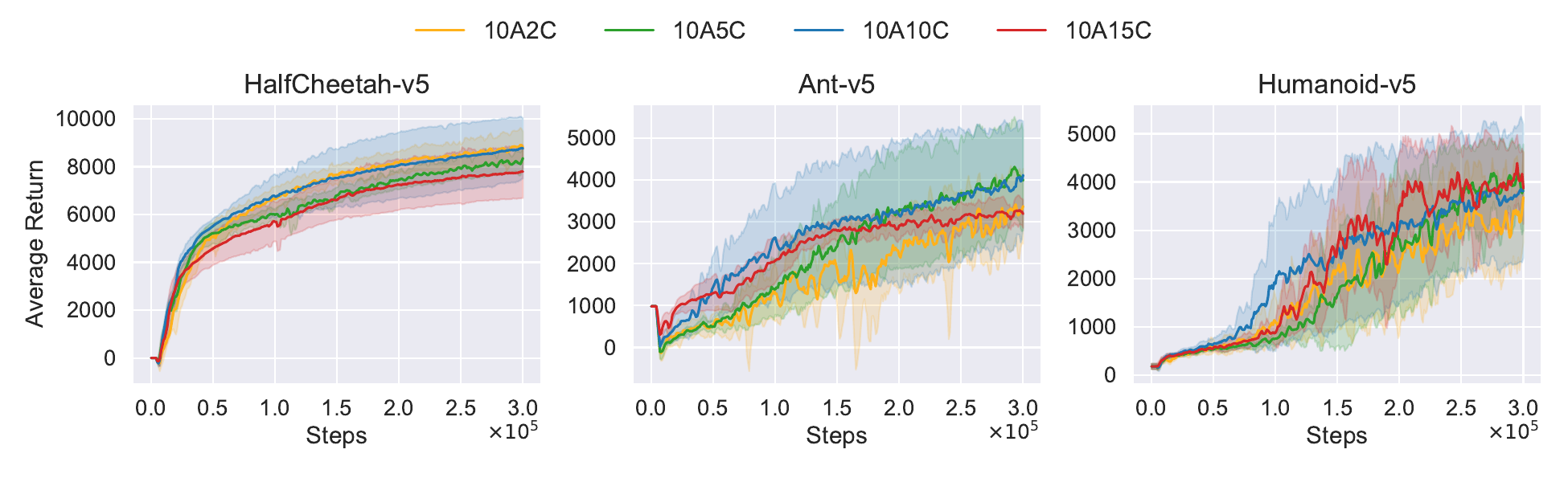}\caption{\label{fig:Average-return-actors-critics}Average return against environmental
steps for MAMC with different number of actors $N_{A}\in\{2,5,10,15\}$
and critics $N_{C}\in\{2,5,10,15\}$ on HalfCheetah-v5, Ant-v5, and
Humanoid-v5 over five trials}
\end{figure}
\end{itemize}

\end{document}